%% file: main.tex
\documentclass{article}
\usepackage{iclr2027_conference,times}
\usepackage[T1]{fontenc}
\input{math_commands.tex}

\usepackage{amsmath,amssymb,amsthm,bm,mathtools}
\usepackage{booktabs,multirow,graphicx,subcaption,float}
\usepackage{placeins}
\usepackage{algorithm}
\usepackage[noend]{algpseudocode}
\usepackage{microtype,xcolor,url,hyperref}
\definecolor{linkblue}{RGB}{0,72,135}
\hypersetup{
  colorlinks=true,
  linkcolor=linkblue,
  citecolor=linkblue,
  urlcolor=linkblue,
  bookmarksdepth=2
}
\renewcommand{\eqref}[1]{\hyperref[#1]{\textup{(\ref*{#1})}}}
\usepackage{tikz}
\usetikzlibrary{arrows.meta,positioning,fit,calc}

\newcommand{\sg}{\operatorname{sg}}

\newcommand{\Sphere}{\mathbb S}

\newtheorem{proposition}{Proposition}
\newtheorem{lemma}[proposition]{Lemma}
\newtheorem{remark}{Remark}
\newtheorem{theorem}{Theorem}
\newtheorem*{mytheorem}{\protect\hyperref[thm:population-error]{Theorem~\protect\ref*{thm:population-error}}}
\theoremstyle{definition}
\newtheorem{definition}{Definition}
\newtheorem{assumption}{Assumption}
\theoremstyle{plain}

\title{Learning a Flow to Self-Supervised Representations}

\author{
\begin{tabular}{c}
\textbf{Yuling Jiao\textsuperscript{1} \quad
Wensen Ma\textsuperscript{2} \quad
Houduo Qi\textsuperscript{3,2} \quad
Defeng Sun\textsuperscript{2}}
\end{tabular}
}

\iclrfinalcopy

\begin{document}
\maketitle
\lhead{Preprint}
\begingroup
\makeatletter
\renewcommand{\thefootnote}{}
\renewcommand{\@makefntext}[1]{\noindent #1}
\footnotetext{%
\raggedright
\textsuperscript{1}School of Artificial Intelligence, National Center for Applied Mathematics in Hubei, Hubei Key Laboratory of Computational Science, Wuhan University, Wuhan, China.\\
\textsuperscript{2}Department of Applied Mathematics, The Hong Kong Polytechnic University, Hung Hom, Kowloon, Hong Kong SAR, China.\\
\textsuperscript{3}Department of Data Science and Artificial Intelligence, The Hong Kong Polytechnic University, Hung Hom, Kowloon, Hong Kong SAR, China.}
\addtocounter{footnote}{-1}
\makeatother
\endgroup

\begin{abstract}
We introduce Flow-Based Distribution Matching (FBDM), which learns
self-supervised representations through non-adversarial velocity regression.
The encoder and auxiliary field are trained jointly by regressing analytic
velocities along spherical paths from learned representations to reference targets.
Both augmented views share a reference target, with a capacity
limit on each reference center and an explicit view-alignment loss.
A structured reference supports more components than the auxiliary flow
dimension while retaining geometric separation.
The auxiliary field is discarded after pretraining; only the backbone is
retained downstream.
On three ResNet-18 benchmarks with matched auxiliary parameter counts, FBDM
provides 2.17--2.36$\times$ higher per-epoch training throughput than critic-based Distribution Matching (DM),
with about 0.21~GiB additional peak allocated GPU memory.
Experiments show that FBDM is competitive with existing SSL methods from
CIFAR to ImageNet and performs comparably to DM on their shared benchmarks.
Under stated conditions, our population analysis bounds nearest-centroid
classification error on the initial normalized pretraining representation by
the FBDM loss.
\end{abstract}

\input{sections/introduction}
\input{sections/related_work}
\input{sections/method}
\input{sections/experiments}
\input{sections/theory}
\input{sections/limitations_conclusion}

\clearpage
\section*{Reproducibility Statement}
The appendix gives the complete objective, assignment procedure, clock,
augmentation and optimization settings, initialization, evaluation
rules, and the distinction between online screening and fixed evaluation. The
public DM code used for the matched training-cost comparison is available at
\url{https://github.com/vincen-github/DM}. The FBDM code is available at
\url{https://github.com/vincen-github/FBDM}.

\section*{AI Use Statement}
The authors developed the research roadmap, initial FBDM idea, and initial
algorithmic implementation. Throughout the project, they set the conceptual
direction, proposed and evaluated experimental directions, identified the
central weakness to address in the theoretical analysis, and specified the
goal of relating the pretraining objective to downstream misclassification.

Within this author-defined framework, OpenAI's ChatGPT and Codex served as
substantial research-assistance tools. They helped develop and implement the
theoretical details, including the population-level analysis, intermediate
mathematical arguments and bounds, and proof drafts toward the authors'
specified goal. They also assisted with algorithmic refinement, experimental
design, implementation and review of training and evaluation code, debugging,
interpretation of intermediate experimental results, literature search,
manuscript drafting and revision, figure and table preparation, and LaTeX.

DeepSeek additionally assisted with editorial and technical critique, citation
checking, and drafting selected passages.

The authors critically examined AI-assisted suggestions, derivations, and
experimental proposals through iterative discussion; chose which ideas and
experiments to pursue; verified the implementations, proofs, citations, and
experimental outputs; and made all final scientific and presentation
decisions. The authors take full responsibility for the accuracy, integrity,
and conclusions of the method, theoretical analysis, experiments, citations,
and reported results.

\bibliography{references}
\bibliographystyle{iclr2027_conference}

\clearpage
\appendix
\input{sections/appendix}

\end{document}

%% file: math_commands.tex
\usepackage{amsmath,amsfonts,bm}

\def\eqref#1{equation~\ref{#1}}

\def\1{\bm{1}}

\DeclareMathAlphabet{\mathsfit}{\encodingdefault}{\sfdefault}{m}{sl}
\SetMathAlphabet{\mathsfit}{bold}{\encodingdefault}{\sfdefault}{bx}{n}



%% file: sections/introduction.tex
\section{Introduction}
\label{sec:introduction}
Self-supervised learning (SSL) learns transferable representations without
labels, but must avoid collapse, in which all images receive the same
embedding. SSL must preserve visual variation across images while aligning
their augmented views. Representative approaches use negative samples
\citep{chen2020simple,he2020momentum}, introduce asymmetry
\citep{grill2020bootstrap,chen2021exploring}, or constrain feature statistics
\citep{zbontar2021barlow,bardes2022vicreg}.

Distribution Matching (DM) specifies the global representation geometry by
matching an encoder-induced law to a designed reference
\citep{jiao2026distribution}. Alongside view alignment, its matching objective is
\begin{equation}
\min_f W_1(f_\#P_{\mathcal A},\nu),
\label{eq:dm-matching-intro}
\end{equation}
where $f$ is the representation map, $P_{\mathcal A}$ the augmented-image
law, $f_\#P_{\mathcal A}$ its pushforward through $f$, and $\nu$ the reference
law. By Kantorovich--Rubinstein duality \citep[Remark~6.5]{villani2009optimal}, for $\mu=f_\#P_{\mathcal A}$,
\begin{equation}
W_1(\mu,\nu)=\sup_{\|g\|_{\mathrm{Lip}}\leq1}
\left\{\mathbb E_{R\sim\nu}g(R)-\mathbb E_{Z\sim\mu}g(Z)\right\}.
\label{eq:dm-dual-intro}
\end{equation}
The critic-based implementation used here alternates encoder minimization with
critic maximization, using five critic updates per encoder update
\citep{arjovsky2017wasserstein,jiao2026distribution}.
The encoder and critic optimize opposing objectives, and their updates
change each other's optimization landscape, yielding coupled dynamics.
Oscillations and convergence difficulties in adversarial training for
generative modeling have been studied
\citep{mescheder2017numerics,mescheder2018training}.
This adversarial optimization is computationally expensive, as reflected in the
matched V100 epoch times (Table~\ref{tab:systems}).
In our experiments, we repeatedly observed downstream-accuracy deterioration
when tuning the learning rate or weight decay in DM. This sensitivity makes
hyperparameter tuning difficult and costly
(Appendix~\ref{app:dm-diagnostics}, Figure~\ref{fig:dm-diagnostics}).

Flow matching is widely used in generative modeling to learn transport
between distributions \citep{lipman2023flow,chen2024geometries}.
We introduce Flow-Based Distribution Matching (FBDM), which brings this idea
to self-supervised learning. In standard generative applications, the
starting distribution is typically fixed, such as a Gaussian prior.
In contrast, FBDM learns the starting representations themselves.
The learned velocity field provides an auxiliary learning signal during
pretraining and is discarded afterward; only the backbone is retained for
downstream tasks (Figure~\ref{fig:overview}).
FBDM replaces adversarial encoder--critic optimization with cooperative
velocity regression, avoiding the coupled adversarial dynamics and repeated
critic updates, thereby simplifying optimization and improving pretraining
throughput.
Across repeated FBDM experiments in this study, we did not observe the pronounced
downstream-accuracy deterioration encountered when tuning DM
(Appendix~\ref{app:dm-diagnostics}, Figure~\ref{fig:dm-diagnostics}).

\input{figures/overview}

FBDM achieves $61.41\%$ ImageNet-1K linear top-1 after 100 ResNet-50
pretraining epochs, exceeding direct-target regression controls under the
common protocol (Table~\ref{tab:direct-regression-main}).
Across four ResNet-18 benchmarks,
FBDM achieves competitive SSL accuracy and is among the fastest methods in
training speed
(Table~\ref{tab:ssl-comparison} and Figure~\ref{fig:systems-all}).
Matched V100 comparisons on three of these benchmarks also show
2.17--2.36$\times$ higher throughput than DM with about 0.21~GiB additional
peak allocated memory (Section~\ref{sec:experiments}).

We also provide a theoretical analysis of FBDM, focusing on three questions:
(i) how an auxiliary flow helps learn useful downstream representations
despite being discarded after pretraining;
(ii) when matching separated reference components yields semantic separation
in the starting representations; and (iii) how this separation and view alignment
support downstream classification (Section~\ref{sec:theory}).

Our contributions are:
\begin{itemize}
\item We introduce FBDM, a flow-based SSL framework that jointly learns source
representations and an auxiliary field through shared-target velocity regression.
\item A conditional population guarantee relates FBDM loss and augmentation
quality to nearest-centroid error on the initial normalized encoder--projector
representation.
\item Experiments show competitive frozen-backbone representation quality and
measured training efficiency, including $2.17$--$2.36\times$ higher throughput
than DM in three matched V100 comparisons.
\end{itemize}

%% file: figures/overview.tex
\begin{figure}[!t]
\centering
\resizebox{\linewidth}{!}{\input{figures/fbdm-figure1-inside-crosses-v3.tikz}}
\caption{\textbf{FBDM pretraining.} Augmented views pass through a backbone and projector to define starting representations. Detached target selection assigns both views one shared reference endpoint, defining spherical paths and analytic velocity targets. Velocity regression jointly trains the representations and auxiliary field, alongside direct view alignment of the starting representations. The two orange modules marked $\times$ are discarded after pretraining; only the backbone supplies downstream features.}
\label{fig:overview}
\end{figure}

%% file: figures/fbdm-figure1-inside-crosses-v3.tikz
\begin{tikzpicture}[
  font=\fontsize{8}{9.5}\selectfont,
  box/.style={draw=blue!45!black,rounded corners=2pt,minimum height=9mm,
    align=center,fill=blue!4,inner sep=3pt,line width=.55pt},
  aux/.style={box,draw=orange!65!black,fill=orange!7},
  arr/.style={-{Latex[length=1.55mm]},line width=.65pt,draw=black!70},
  flow/.style={-{Latex[length=1.7mm]},line width=1pt,draw=teal!75!black},
  discard/.style={text=red!65!black,font=\fontsize{14}{14}\selectfont,
    inner sep=0pt}]
  \path[use as bounding box] (-.05,-.68) rectangle (13.92,3.88);

  \node[anchor=east,text=black!80] (discardlegend) at (13.8,3.59)
    {Discard after pretraining};
  \node[discard] at ($(discardlegend.west)+(-.22,0)$) {$\times$};

  \node[box,minimum width=8mm] (image) at (.45,.95) {Image};
  \node[box,minimum width=12mm] (views) at (1.8,.95)
    {Augmented\\views};
  \node[box,minimum width=15mm] (backbone) at (3.6,.95) {Backbone};
  \node[aux,minimum width=15mm] (projector) at (5.45,.95) {Projector};
  \draw[arr] (image) -- (views);
  \draw[arr] (views) -- (backbone);
  \draw[arr] (backbone) -- (projector);

  \node[box,draw=black!45,fill=black!2,minimum width=27mm,
    minimum height=8mm] (selection) at (5.45,2.95)
    {Shared target selection\\consensus + capacity};
  \node[box,draw=purple!50!black,fill=purple!4,
    minimum width=23mm,minimum height=8mm] (reference) at (9.1,2.95)
    {Fixed reference\\sample one target};
  \draw[arr] (projector.north) -- (selection.south)
    node[midway,left,align=right,font=\fontsize{7}{8}\selectfont,
      text=black!65,inner xsep=3pt] {stop\\gradient};
  \draw[arr] (selection) -- (reference);

  \coordinate (center) at (7.95,.95);
  \draw[draw=blue!55!black,line width=.8pt,fill=blue!3]
    (center) circle (1.02);
  \draw[gray!45] ($(center)+(-1.02,0)$)
    arc[start angle=180,end angle=360,x radius=1.02cm,y radius=.31cm];
  \draw[gray!35,dashed] ($(center)+(1.02,0)$)
    arc[start angle=0,end angle=180,x radius=1.02cm,y radius=.31cm];
  \draw[gray!45] ($(center)+(0,-1.02)$)
    .. controls ($(center)+(.4,-.5)$) and ($(center)+(.4,.5)$)
    .. ($(center)+(0,1.02)$);
  \draw[gray!35,dashed] ($(center)+(0,-1.02)$)
    .. controls ($(center)+(-.4,-.5)$) and ($(center)+(-.4,.5)$)
    .. ($(center)+(0,1.02)$);
  \coordinate (starta) at ($(center)+(-.63,-.35)$);
  \coordinate (startb) at ($(center)+(-.48,-.58)$);
  \coordinate (target) at ($(center)+(.53,.4)$);
  \draw[flow] (starta) .. controls ($(center)+(-.2,-.68)$) and
    ($(center)+(.3,-.05)$) .. (target);
  \draw[flow,draw=teal!55!black] (startb) .. controls ($(center)+(.15,-.7)$) and
    ($(center)+(.5,-.05)$) .. (target);
  \fill[blue!65!black] (starta) circle (1.2pt);
  \fill[blue!65!black] (startb) circle (1.2pt);
  \fill[purple!80!black] (target) circle (1.5pt);
  \draw[-{Latex[length=1.6mm]},line width=.8pt,red!65!black]
    ($(center)+(.02,-.18)$) -- ($(center)+(.39,.16)$);
  \node[anchor=north,text=blue!50!black] at ($(center)+(0,-1.1)$)
    {Spherical flow};
  \draw[arr] (projector.east) -- ($(center)+(-1.02,0)$);
  \draw[arr,draw=purple!65!black] (reference.south) |- (target);

  \node[aux,minimum width=19mm] (velocity) at (10.25,.95)
    {Auxiliary\\velocity field};
  \node[box,draw=teal!60!black,fill=teal!5,minimum width=21mm]
    (loss) at (12.65,.95) {Velocity regression\\+ view alignment};
  \draw[arr] ($(center)+(1.02,0)$) -- (velocity);
  \draw[arr] (velocity) -- (loss);

  \node[discard,font=\fontsize{10}{10}\selectfont,anchor=north east]
    at ($(projector.north east)+(-1.5pt,-1.5pt)$) {$\times$};
  \node[discard,font=\fontsize{10}{10}\selectfont,anchor=north east]
    at ($(velocity.north east)+(-1.5pt,-1.5pt)$) {$\times$};
\end{tikzpicture}

%% file: sections/related_work.tex
\section{Related Work}
\label{sec:related}
SSL theory studies latent-class transfer, alignment and uniformity, spectral
structure, and generalization
\citep{saunshi2019theoretical,wang2020understanding,haochen2021provable,
haochen2022beyond,huang2023towards,duan2025advssl}.
DM connects representation quality to explicit Wasserstein matching
\citep{jiao2026distribution}; flow matching learns transport by velocity
regression \citep{lipman2023flow,liu2023rectified,fukumizu2025flow}, including manifold geodesics
\citep{chen2024geometries}.

REPA observes that diffusion models learn meaningful representations, but their
classification features lag behind those learned by SSL. It aligns them with
pretrained SSL features to improve generation \citep{yu2025repa}.
Self-Flow learns representations within the generative framework without
relying on an external pretrained encoder \citep{chefer2026selfflow}; SenFlow
jointly trains an encoder and a conditional flow generator with a signal
generation objective \citep{ukita2026senflow}.
FBDM instead uses reference-directed flow to learn starting representations
for downstream recognition, without image reconstruction,
generation, or a pretrained teacher.
Appendix~\ref{app:related} gives a broader discussion.

%% file: sections/method.tex
\section{Flow-Based Distribution Matching}
\label{sec:method}

FBDM learns starting representations through velocity regression toward shared
reference targets.

\paragraph{Notation.}
Vectors and matrices are bold; $\mathbf X$, $\mathbf Z$, and $\mathbf R$
denote random vectors.
We reserve $d$ for the input-image dimension and $d^\star$ for the auxiliary
flow-space dimension. Write $[N]=\{1,\ldots,N\}$,
$\Sphere^{q-1}=\{\mathbf u\in\mathbb R^q:\|\mathbf u\|_2=1\}$, and
$\mathbf I_q,\mathbf1$ for the identity matrix and all-ones vector.
For a set $\mathcal S$, $\operatorname{Unif}(\mathcal S)$ is its
uniform law; $\sg$ is stop-gradient and $\hookrightarrow$ an injective map.

\subsection{Representations and a Flexible Reference}
\label{subsec:reference}
Given unlabeled images $\mathcal D_N=\{\mathbf x_i\}_{i=1}^N\subset\mathbb R^d$,
each iteration draws $B$ images. For each $\mathbf x^{(n)}$, draw two views
$\mathbf x_1^{(n)},\mathbf x_2^{(n)}$ independently from the augmentation kernel
$\mathcal A(\,\cdot\mid\mathbf x^{(n)})$. Let $f_\theta$ be the backbone and
$h_\psi$ its projection head, with parameters $\theta$ and $\psi$.
Assignment compares directions, so we normalize their output:
\begin{equation}
\mathbf z_j^{(n)}
=\frac{h_\psi(f_\theta(\mathbf x_j^{(n)}))}
{\|h_\psi(f_\theta(\mathbf x_j^{(n)}))\|_2}
\in\Sphere^{d^\star-1},\qquad j\in\{1,2\}.
\label{eq:representation}
\end{equation}
We next define the destination law and transport signal.

\paragraph{More components without increasing the flow-space dimension.}
We seek compact representations that preserve semantic structure so related
images can group. Let $K'$ count label-free reference components, $K$ latent
semantic classes (Section~\ref{sec:theory}), and $d^\star$ denote the auxiliary
flow-space dimension. $K'$ need not equal $K$: choosing $K'>K$ can accommodate
within-class diversity, such as a close-up of a horse's head, a full-body
image, or a horse carrying a rider, all belonging to the horse class.
The coordinate reference
used in DM has $\mathbf c_k^{\mathrm{DM}}=s_k\mathbf e_{\pi(k)}$, where
$\mathbf e_i$ is the $i$th standard basis vector in $\mathbb R^{d^\star}$, with independent
$s_k\sim\operatorname{Unif}(\{-1,1\})$ and
$\pi:[K']\hookrightarrow[d^\star]$ \citep{jiao2026distribution}.
Injectivity implies $K'\le d^\star$, limiting reference granularity at a
fixed flow dimension. To support richer semantic structure in Tiny ImageNet
and ImageNet-1K, we relax this coupling with an ETF-inspired construction
of $K'$ unit centers
$\mathcal C=\{\mathbf c_k\}_{k=1}^{K'}$ with $K'>d^\star$.
Figure~\ref{fig:dm-etf-reference} illustrates this decoupling.

\input{figures/dm_etf_reference}

\paragraph{From an ideal geometry to a compact reference.}
An equiangular tight frame (ETF) supplies an ideal separated geometry
(Appendix~\ref{app:proof-welch}), but need not exist for arbitrary
$(K',d^\star)$. We start with the exact simplex in $K'-1$ dimensions,
whose Gram matrix is
\begin{equation}
\begin{gathered}
\mathbf G_\Delta=\frac{K'}{K'-1}
\left(\mathbf I_{K'}-\frac{\mathbf1\mathbf1^\top}{K'}\right)
\in\mathbb R^{K'\times K'},\\
\text{with entries}\quad (\mathbf G_\Delta)_{k\ell}=
\begin{cases}
1, & k=\ell,\\
-\dfrac{1}{K'-1}, & k\ne\ell,
\end{cases}
\qquad k,\ell\in[K'].
\end{gathered}
\label{eq:simplex-gram}
\end{equation}
For $2\le d^\star<K'-1$, retain $d^\star$ leading orthonormal eigenvectors
and eigenvalues in $\mathbf U_{d^\star}$ and $\bm\Lambda_{d^\star}$:
\begin{equation}
\mathbf Y=\mathbf U_{d^\star}\bm\Lambda_{d^\star}^{1/2}
\in\mathbb R^{K'\times d^\star}.
\label{eq:projected-coordinates}
\end{equation}
Row normalization gives spherical centers:
\[
\mathbf c_k=\frac{\mathbf Y_{k,:}}{\|\mathbf Y_{k,:}\|_2},
\qquad k\in[K'].
\]
Choose the retained eigenspace to give nonzero rows and distinct normalized
directions. The construction recovers the exact simplex at $d^\star=K'-1$;
in lower dimensions the normalized reference is ETF-inspired. The spectral
approximation guarantee applies before row normalization
(Lemma~\ref{lem:spectral}). Appendix~\ref{app:spectral-construction} gives the
construction details and Algorithm~\ref{alg:reference-centers}.

To retain local variation rather than force all images in a component to
the same point, perturb its center by a spherical direction, with $0\le\varepsilon<1$:
\begin{equation}
\begin{gathered}
Q_k^\varepsilon=\operatorname{Law}\!\left(
\frac{\mathbf c_k+\varepsilon\bm\eta}
{\|\mathbf c_k+\varepsilon\bm\eta\|_2}\right),
\qquad \bm\eta\sim\operatorname{Unif}(\Sphere^{d^\star-1}),\\
\nu_\varepsilon=\sum_{k=1}^{K'}w_kQ_k^\varepsilon,
\qquad w_k>0,\quad \sum_{k=1}^{K'}w_k=1.
\end{gathered}
\label{eq:reference-law}
\end{equation}
We use the uniform reference $w_k=1/K'$; assigned-target frequencies may differ.
Here $K'$ controls template granularity and $\varepsilon$ local spread;
components are label-free templates, not prescribed ground-truth classes.

\subsection{Two-View Capacitated Assignment}
\label{subsec:assignment}
Flow matching must pair each view with a destination. To align a positive pair
implicitly, we choose one shared destination using its detached consensus:
\begin{equation}
\mathbf q^{(n)}
=\frac{\sg(\mathbf z_1^{(n)})+\sg(\mathbf z_2^{(n)})}
{\|\sg(\mathbf z_1^{(n)})+\sg(\mathbf z_2^{(n)})\|_2}.
\label{eq:consensus}
\end{equation}
Detachment keeps the combinatorial assignment outside the gradient path.

A directly balanced construction assumes $K'\mid B$, meaning that $K'$ divides
$B$, and allocates
$m_0=B/K'$ images to each center. However, minibatches need not be semantically
balanced, and fixed per-center quotas may force images sharing similar
semantics toward incompatible components.
To relax this fixed quota, we use an epoch-indexed capacity
$m_{\mathcal D_N}(e)>B/K'$, where $e$ denotes the training epoch, and select
only $B$ of the available slots:
\begin{equation}
\mathcal P_{\mathcal D_N}(e)=[K']\times[m_{\mathcal D_N}(e)],
\qquad \kappa:[K']\times\mathbb N\to[K'],\quad \kappa((k,\ell))=k.
\label{eq:capacity}
\end{equation}
The resulting capacitated optimal transport assignment maximizes total similarity:
\begin{equation}
\alpha^*\in\arg\max_{\alpha:[B]\hookrightarrow\mathcal P_{\mathcal D_N}(e)}
\sum_{n=1}^B(\mathbf q^{(n)})^\top\mathbf c_{\kappa(\alpha(n))}
\label{eq:hungarian}
\end{equation}
It caps occupancy without enforcing equality. Draw
$\bm\eta^{(n)}\sim\operatorname{Unif}(\Sphere^{d^\star-1})$ and set
\begin{equation}
\mathbf r^{(n)}
=\frac{\mathbf c_{\kappa(\alpha^*(n))}+\varepsilon\bm\eta^{(n)}}
{\|\mathbf c_{\kappa(\alpha^*(n))}+\varepsilon\bm\eta^{(n)}\|_2}.
\label{eq:target}
\end{equation}
The resulting $\mathbf r^{(n)}$ is the shared endpoint of the two analytic
paths starting at $\mathbf z_j^{(n)}$. This gives the views a common
target-directed regression task.
The capacity schedule controls per-center availability;
Appendix~\ref{app:configs}, Table~\ref{tab:dataset-capacity}, gives the
dataset-specific choices.

\subsection{Spherical Transport and Its Learning Signal}
\label{subsec:flow}
Because representations and destinations lie on the sphere, a shortest
great-circle path is a natural interpolation \citep{chen2024geometries}.
For non-antipodal $\mathbf z,\mathbf r$, let
$\omega=\arccos(\mathbf z^\top\mathbf r)$ and define progress $s\in[0,1]$ by
\begin{equation}
\bm\gamma_{\mathbf z,\mathbf r}(s)
=\frac{\sin((1-s)\omega)}{\sin\omega}\mathbf z
+\frac{\sin(s\omega)}{\sin\omega}\mathbf r.
\label{eq:slerp}
\end{equation}
The $\omega=0$ case is understood by continuity.
Appendix~\ref{app:derive-slerp} derives the spherical path and its tangent
velocity. Table~\ref{tab:linear-path-ablation} compares spherical and linear paths
(Appendix~\ref{app:linear-interpolation}).

Fix a training epoch $e$ and suppress its dependence in the notation.
Let $\mathbf Z_0$ be a starting representation and $\mathbf R$ its assigned
endpoint from Section~\ref{subsec:assignment}. Write $f$ for the normalized
backbone--projector map in \eqref{eq:representation} and denote the paired law
by $\Pi_f=\operatorname{Law}(\mathbf Z_0,\mathbf R)$. For
$(\mathbf Z_0,\mathbf R)\sim\Pi_f$, a continuously differentiable increasing
clock $s(t)$ maps network time $t$ to path progress, with $s(0)=0,s(1)=1$:
\begin{equation}
\mathbf Z_t=\bm\gamma_{\mathbf Z_0,\mathbf R}(s(t)),
\qquad p_t=\operatorname{Law}(\mathbf Z_t).
\label{eq:clocked-law}
\end{equation}
Its analytic velocity is the signal to be regressed:
\begin{equation}
\mathbf U_t=\frac{\mathrm d\mathbf Z_t}{\mathrm dt}
=\dot s(t)\left.\partial_s\bm\gamma_{\mathbf Z_0,\mathbf R}(s)
\right|_{s=s(t)}.
\label{eq:clocked-target}
\end{equation}
The conditional mean
$\bm v^\star(\mathbf z,t)=\mathbb E[\mathbf U_t\mid\mathbf Z_t=\mathbf z]$
realizes the distributional path through
\begin{equation}
\partial_t p_t(\mathbf z)+\nabla\!\cdot
\left(p_t(\mathbf z)\bm v^\star(\mathbf z,t)\right)=0.
\label{eq:continuity}
\end{equation}
Here $\nabla\!\cdot$ is divergence; when $p_t$ has no ambient density,
\eqref{eq:continuity} is understood in the standard weak sense.
We fit a velocity network $\bm v_\phi$, with parameters $\phi$, to these
pair-specific targets. Sampling $T$ independently of the endpoints from a
density $q_T$ gives the regression risk \citep{lipman2023flow}:
\begin{equation}
\mathcal R_{\mathrm{CFM}}(\phi;q_T)
=\frac1{d^\star}\mathbb E_{\Pi_f,T}
\left[\|\bm v_\phi(\mathbf Z_T,T)-\mathbf U_T\|_2^2\right],
\qquad T\sim q_T,
\label{eq:pathwise-cfm-risk}
\end{equation}

\paragraph{Time reparameterization and uniform-progress sampling.}
For semantic encoding, reaching an appropriate reference component matters
more than precise positioning within it. We therefore choose clocks with
decreasing $\dot s(t)$. Through \eqref{eq:clocked-target}, this strengthens
component-directed motion early and attenuates within-component refinement
late without changing the path. For a shape parameter $a>0$, the rational
schedule used on several datasets has $\dot s(t)=(1+a)/(1+at)^2$.
The reported rational and exponential choices appear in
Table~\ref{tab:time-schedules} and follow the distance-scheduling perspective
of \citet{chen2024geometries}. Uniformly sampling $t$ would overrepresent the
slow late segment, so inverse-transform sampling instead draws
$S_j^{(n)}\sim\operatorname{Unif}([0,1])$ and sets
\begin{equation}
T_j^{(n)}=s^{-1}(S_j^{(n)}).
\label{eq:path-uniform}
\end{equation}
This gives $q_T(t)=\dot s(t)$ and uniform coverage of path progress; the
decreasing clock speed gives larger velocity targets early along each path.
The network takes $(\mathbf Z_j^{(n)},T_j^{(n)})$ and predicts
$\mathbf U_j^{(n)}$, where
\begin{equation}
\mathbf Z_j^{(n)}=\bm\gamma_{\mathbf z_j^{(n)},\mathbf r^{(n)}}(S_j^{(n)}),
\qquad
\mathbf U_j^{(n)}=\dot s(T_j^{(n)})
\partial_s\bm\gamma_{\mathbf z_j^{(n)},\mathbf r^{(n)}}(S_j^{(n)}).
\label{eq:conditional-pair}
\end{equation}

\subsection{Objective, Updates, and Inference}
\label{subsec:empirical-risk}
We combine shared-target velocity regression with an explicit alignment
penalty on the starting representations:
\begin{align}
\widehat{\mathcal R}_{\mathrm{FM},B}
&=\frac1{2Bd^\star}\sum_{n=1}^B\sum_{j=1}^2
\|\bm v_\phi(\mathbf Z_j^{(n)},T_j^{(n)})-\mathbf U_j^{(n)}\|_2^2,
\label{eq:fm}\\
\widehat{\mathcal R}_{\mathrm{align},B}
&=\frac1{Bd^\star}\sum_{n=1}^B
\|\mathbf z_1^{(n)}-\mathbf z_2^{(n)}\|_2^2,
\label{eq:alignment}\\
\widehat{\mathcal R}_B
&=\widehat{\mathcal R}_{\mathrm{FM},B}
+\lambda\widehat{\mathcal R}_{\mathrm{align},B},\qquad\lambda\ge0.
\label{eq:objective}
\end{align}
Both the interpolated states and analytic targets depend on the encoder--projector
outputs, so velocity regression updates the starting representations together
with the field. Target assignment remains detached.
\input{figures/fbdm_algorithm}

Algorithm~\ref{alg:fbdm} uses detached assignments.
Dataset-specific training refinements and update rules are described in
Appendix~\ref{app:grad-scale}. After pretraining, only the backbone feature
$f_\theta(\mathbf x)$ is used downstream.

%% file: figures/dm_etf_reference.tex
\begin{figure}[!t]
\centering
\includegraphics[width=\linewidth]{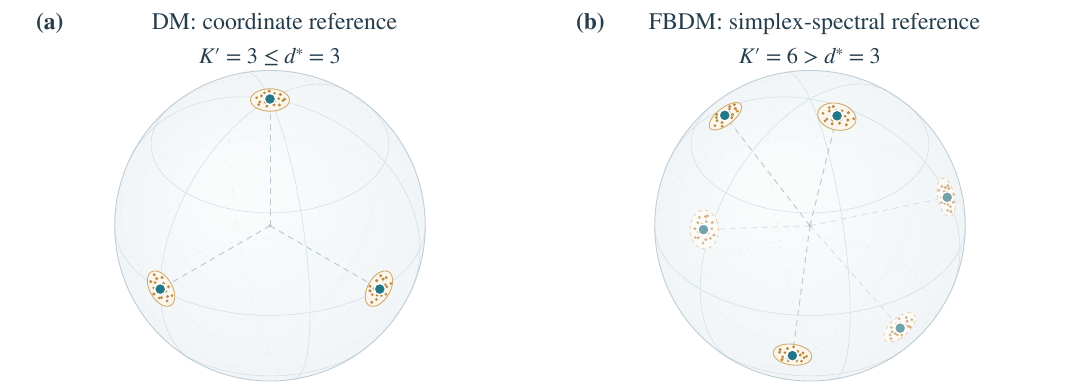}
\caption{\textbf{Flexible spherical references.} (a) DM's coordinate construction
requires $K'\le d^\star$ \citep{jiao2026distribution}; (b) our simplex-spectral
construction permits $K'>d^\star$. Clouds show local perturbations around
each center. Proposition~\ref{prop:welch} states the angular benchmark.}
\label{fig:dm-etf-reference}
\end{figure}

%% file: figures/fbdm_algorithm.tex
\begin{algorithm}[H]
\caption{Core FBDM pretraining loop.}
\label{alg:fbdm}
\begin{algorithmic}[1]
\Require $\mathcal D_N$, $\mathcal A$, $\mathcal C$, $\varepsilon$, $\lambda$,
$m_{\mathcal D_N}(e)$, $s(t)$; networks $f_\theta$, $h_\psi$, $\bm v_\phi$
\For{each epoch $e$ and minibatch $\mathcal B\subset\mathcal D_N$}
  \State Draw two views and compute $\mathbf z_j^{(n)}$ by
  \eqref{eq:representation}
  \State Form $\mathbf q^{(n)}$ and solve the capacitated assignment
  \eqref{eq:hungarian}
  \State Draw one endpoint $\mathbf r^{(n)}$ by \eqref{eq:target} and share it
  across both views
  \State Draw $S_j^{(n)}$, obtain $T_j^{(n)}$, and form
  $(\mathbf Z_j^{(n)},\mathbf U_j^{(n)})$ by
  \eqref{eq:path-uniform}--\eqref{eq:conditional-pair}
  \State Evaluate \eqref{eq:objective} and update $(\theta,\psi,\phi)$
\EndFor
\State \textbf{return} $f_\theta$
\end{algorithmic}
\end{algorithm}

%% file: sections/experiments.tex
\section{Experiments}
\label{sec:experiments}

We evaluate FBDM's representation quality and training efficiency against
existing SSL methods across five datasets, and compare its training cost with DM
under matched auxiliary-network sizes. We also compare velocity-based training
with direct regression to the assigned targets. Implementation details and
additional ablations are given in Appendices~\ref{app:configs} and~\ref{app:ablations},
respectively. The FBDM code is available at
\url{https://github.com/vincen-github/FBDM}.

\subsection{Comparison with SSL Methods}

\paragraph{CIFAR-10, CIFAR-100, STL-10, and Tiny ImageNet.}
We evaluate frozen ResNet-18 features using linear classification and cosine
$5$-NN (Table~\ref{tab:ssl-comparison}). FBDM pretrains for 1000 epochs and
fits a frozen-backbone linear classifier for 500 epochs. We scale the FM gradient entering the encoder and
projector by $0.75$ to let the velocity field track the evolving representation
geometry (CIFAR-10 ablation: Table~\ref{tab:ablations}). Configurations and
update rules are in Appendix~\ref{app:configs} and Appendix~\ref{app:grad-scale};
Table~\ref{tab:core-configuration} summarizes the ResNet-18 settings.

\begin{table}[!htb]
\centering
\caption{Top-1 accuracy (\%) on CIFAR-10, CIFAR-100, STL-10, and Tiny ImageNet
with ResNet-18. \emph{Linear} evaluates the frozen representation; $5$-NN uses
$k=5$. Non-FBDM results except DM follow the published benchmark
\citep{weng2022investigation}; DM results are taken from its source
paper~\citep{jiao2026distribution}. See Figure~\ref{fig:systems-all} for V100 training times.}
\label{tab:ssl-comparison}
\resizebox{\textwidth}{!}{%
\begin{tabular}{lcccccccc}
\toprule
& \multicolumn{2}{c}{CIFAR-10}
& \multicolumn{2}{c}{CIFAR-100}
& \multicolumn{2}{c}{STL-10}
& \multicolumn{2}{c}{Tiny ImageNet}\\
\cmidrule(lr){2-3}
\cmidrule(lr){4-5}
\cmidrule(lr){6-7}
\cmidrule(lr){8-9}
Method
& Linear & $5$-NN
& Linear & $5$-NN
& Linear & $5$-NN
& Linear & $5$-NN\\
\midrule
\textbf{FBDM}
& $92.37$ & $89.68$
& $66.59$ & $56.74$
& $89.79$ & $86.23$
& $48.55$ & $33.01$\\
\midrule
DM~\citep{jiao2026distribution}
& $92.11$ & $89.17$
& $67.71$ & $56.18$
& $90.22$ & $85.51$
& \multicolumn{2}{c}{--}\\
Barlow Twins~\citep{zbontar2021barlow}
& $88.51$ & $86.53$
& $65.78$ & $55.76$
& $88.36$ & $83.71$
& $47.44$ & $32.65$\\
SimCLR~\citep{chen2020simple}
& $91.80$ & $88.42$
& $66.83$ & $56.56$
& $90.51$ & $85.68$
& $48.84$ & $32.86$\\
BYOL~\citep{grill2020bootstrap}
& $91.73$ & $89.45$
& $66.60$ & $56.82$
& $91.99$ & $88.64$
& $51.00$ & $36.24$\\
SimSiam~\citep{chen2021exploring}
& $90.51$ & $86.82$
& $66.04$ & $55.79$
& $88.91$ & $84.84$
& $48.29$ & $34.21$\\
Shuffled-DBN~\citep{hua2021feature}
& $90.45$ & $88.15$
& $66.07$ & $56.97$
& $89.20$ & $84.51$
& $48.60$ & $32.14$\\
Zero-ICL~\citep{zhang2022zerocl}
& $88.12$ & $86.64$
& $61.91$ & $53.47$
& $86.35$ & $82.51$
& $46.25$ & $32.74$\\
VICReg~\citep{bardes2022vicreg}
& $90.32$ & $88.41$
& $66.45$ & $56.78$
& $90.78$ & $85.72$
& $48.71$ & $33.35$\\
W-MSE~\citep{ermolov2021whitening}
& $91.55$ & $89.69$
& $66.10$ & $56.69$
& $90.36$ & $87.10$
& $48.20$ & $34.16$\\
CW-RGP~\citep{weng2022investigation}
& $91.92$ & $89.54$
& $67.51$ & $57.35$
& $90.76$ & $87.34$
& $49.23$ & $34.04$\\
\bottomrule
\end{tabular}}
\end{table}

\begin{figure}[!htb]
\centering
\includegraphics[width=\linewidth]{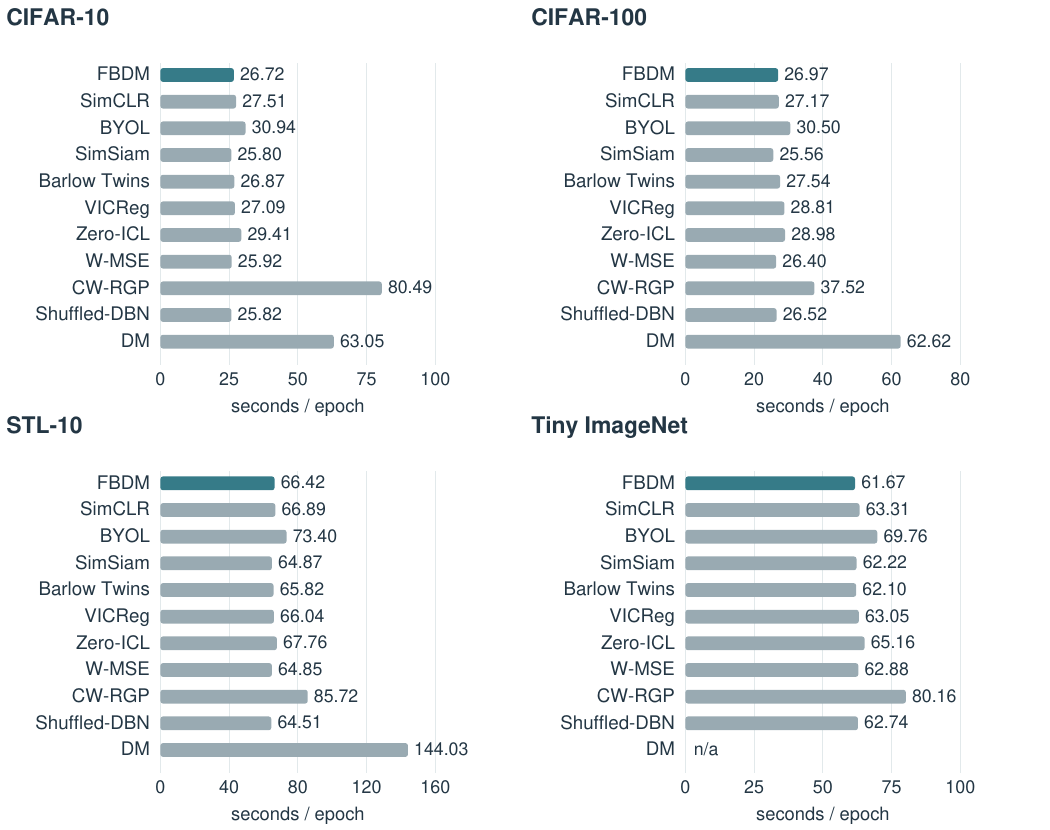}
\caption{V100 training time per epoch for Table~\ref{tab:ssl-comparison}; lower
is better. Panels use separate scales; DM timing is unavailable for Tiny ImageNet.
Values and measurement details: Appendix~\ref{app:systems-protocol}.}
\label{fig:systems-all}
\end{figure}

Table~\ref{tab:ssl-comparison} and Figure~\ref{fig:systems-all} compare frozen-feature
accuracy and per-epoch training cost across the four ResNet-18 benchmarks.
FBDM learns competitive representations, and its epoch time is within 2.9\%
of SimCLR on each of these datasets.

\FloatBarrier
\paragraph{ImageNet-1K.}
We compare ResNet-50 with SimCLR after 100 pretraining epochs at batch size 512
(Table~\ref{tab:imagenet-comparison}). The field-only energy penalty uses $\beta=20$.
Under its conditions, Proposition~\ref{prop:velocity-energy-matching} bounds
the distance between initial-representation and assigned-target distributions
by FM risk and velocity energy, motivating the penalty. Update rules and
ablations are in Appendices~\ref{app:grad-scale} and~\ref{app:imagenet-energy}.

\begin{table}[H]
\centering
\captionsetup{font=small,skip=3pt}
\caption{ImageNet-1K with ResNet-50 and global batch size 512. Linear is frozen-
backbone top-1 (\%) after 100 pretraining epochs; time is minutes per epoch on
four V100 GPUs with AMP, averaged over two epochs (Appendix~\ref{app:systems-protocol}).
SimCLR accuracy is from
\href{https://github.com/AndrewAtanov/simclr-pytorch}{\texttt{AndrewAtanov/simclr-pytorch}}.
}
\label{tab:imagenet-comparison}
\small
\begin{tabular}{lrr}
\toprule
Method & Linear (\%) & Minutes / epoch\\
\midrule
FBDM & $61.41$ & $29.29$\\
SimCLR~\citep{chen2020simple} & $60.14$ & $30.02$\\
\bottomrule
\end{tabular}
\vspace{-0.35em}
\end{table}

\FloatBarrier
\subsection{Comparison with DM}

We compare DM and FBDM on one V100 in FP32 at batch size 512, with matched
auxiliary parameter counts (Table~\ref{tab:auxiliary-parameters}).
Table~\ref{tab:systems} reports their epoch times and speedup ratios;
Appendix~\ref{app:systems-protocol} gives the timing and memory protocol.

\begin{table}[H]
\centering
\captionsetup{font=small,skip=3pt}
\caption{Matched DM--FBDM training cost on one V100 in FP32 at batch size
512. Speedup is the DM-to-FBDM epoch-time ratio. Memory entries are peak
CUDA allocated/reserved memory (GiB) after warm-up.}
\label{tab:systems}
\label{tab:memory}
\small
\setlength{\tabcolsep}{5pt}
\renewcommand{\arraystretch}{0.92}
\begin{tabular}{lrrrrr}
\toprule
& \multicolumn{3}{c}{Time (s/epoch)}
& \multicolumn{2}{c}{Peak CUDA memory (GiB)}\\
\cmidrule(lr){2-4}\cmidrule(lr){5-6}
Dataset & DM & FBDM & \textbf{Speedup} & DM & FBDM\\
\midrule
CIFAR-10 & $63.05$ & $26.72$ & $\bm{2.36\times}$ & $5.07/5.44$ & $5.28/5.79$\\
CIFAR-100 & $62.62$ & $26.97$ & $\bm{2.32\times}$ & $5.07/5.43$ & $5.28/5.79$\\
STL-10 & $144.03$ & $66.42$ & $\bm{2.17\times}$ & $7.35/8.52$ & $7.56/8.71$\\
\bottomrule
\end{tabular}
\end{table}

\FloatBarrier
\subsection{Ablation Studies}
\label{subsec:ablation-studies}

\paragraph{Velocity regression versus direct regression.}
To assess the benefit of velocity-field training with energy regularization,
we compare FBDM with Euclidean and geodesic regression to the same assigned targets
on ImageNet-1K. Both controls retain assignment and alignment but omit the
velocity field and its energy penalty. FBDM leads by $3.39$ and $2.66$ percentage
points, respectively (Table~\ref{tab:direct-regression-main}).
Appendix~\ref{app:direct-regression} gives the objectives and full protocol.
Further ablations of the velocity-energy penalty are in
Appendix~\ref{app:imagenet-energy}.

\begin{table}[H]
\centering
\caption{ImageNet-1K frozen-backbone linear top-1 (\%) after 100 pretraining
epochs. Control details are in Appendix~\ref{app:direct-regression}.}
\label{tab:direct-regression-main}
\small
\begin{tabular}{lc}
\toprule
Pretraining objective & Linear (\%)\\
\midrule
FBDM with velocity energy ($\beta=20$) & $61.41$\\
Direct Euclidean regression & $58.02$\\
Direct geodesic regression & $58.75$\\
\bottomrule
\end{tabular}
\end{table}

Results for spherical versus linear paths, explicit alignment, and other training
refinements are in Appendices~\ref{app:linear-interpolation}
and~\ref{app:alignment-optimization}.

%% file: sections/theory.tex
\section{Theoretical Analysis}
\label{sec:theory}
\subsection{Questions Motivating the Analysis}
We ask how a flow discarded after pretraining shapes downstream representations,
when matching separated reference components yields semantic separation at the
starting representations, and how separation and alignment imply a classification
guarantee.
Under flow regularity, FM risk controls terminal matching error, and transport
identifiability relates this error to class--component correspondence. Bounded
flow expansion transfers terminal separation to the initial class means, with
corrections for within-class spread. Alignment and augmentation quality then
yield the conditional nearest-centroid guarantee.
Appendix~\ref{app:population-proof} gives the full argument.

\subsection{Population Objective and Assumptions}
\label{subsec:population-setup}
We analyze FBDM at the population level to connect its training objective to
downstream classification. Pretraining and classification use the same distribution
$(\mathbf X,Y)\sim P$, with labels hidden during pretraining, $Y\in[K]$,
and $p_k=P(Y=k)>0$; write $p_{\min}=\min_k p_k$.
We treat the normalized map in \eqref{eq:representation} as a single encoder
$f:\mathcal X\to\Sphere^{d^\star-1}$, where $\mathcal X\subseteq\mathbb R^d$
contains images and their views.

The augmentation kernel $\mathcal A$ is fixed. Given $\mathbf X$, let
$\widetilde{\mathbf X}_1$ and $\widetilde{\mathbf X}_2$ be conditionally
independent draws from $\mathcal A(\,\cdot\mid\mathbf X)$. Denote the law of
one view by $P_{\mathcal A}$ and their joint law by
$\Pi_{\mathcal A}^+=\operatorname{Law}(\widetilde{\mathbf X}_1,
\widetilde{\mathbf X}_2)$. The reference is also fixed: we consider $K'=K$ and
$\nu_\varepsilon=\sum_k w_kQ_k^\varepsilon$, with $w_k>0$ and $\sum_k w_k=1$.
No equal class priors or equality $p_k=w_k$ is required.
Set $\mathcal S_k^\varepsilon=\operatorname{supp}(Q_k^\varepsilon)$ and
\begin{equation}
\Delta_{\mathcal C}=\min_{k\ne\ell}\|\mathbf c_k-\mathbf c_\ell\|_2,
\qquad
r_\varepsilon=\max_k\sup_{\mathbf r\in\mathcal S_k^\varepsilon}
\|\mathbf r-\mathbf c_k\|_2,
\qquad \Delta_{\mathcal C}>2r_\varepsilon.
\label{eq:theory-reference-geometry}
\end{equation}
Thus the reference supports are disjoint. Recall the representation--endpoint
coupling $\Pi_f=\operatorname{Law}(\mathbf Z_0,\mathbf R)$ from
Section~\ref{subsec:flow}. For the population analysis, its source marginal is
the encoded augmentation law
$\operatorname{Law}(\mathbf Z_0)=f_\#P_{\mathcal A}$, and its endpoint marginal
is $\operatorname{Law}(\mathbf R)=\nu_\varepsilon$. The latent label $Y$
remains on the underlying population space but is never used during pretraining.
With paths \eqref{eq:clocked-law}, targets \eqref{eq:clocked-target}, and
independent $T\sim q_T$, define
\begin{align}
\mathcal R_{\mathrm{FM}}
&=\frac1{d^\star}\mathbb E_{\Pi_f,T}
[\|\bm v_\phi(\mathbf Z_T,T)-\mathbf U_T\|_2^2],
\label{eq:population-fm-risk}\\
\mathcal R_{\mathrm{align}}
&=\frac1{d^\star}\mathbb E_{\Pi_{\mathcal A}^+}
[\|f(\widetilde{\mathbf X}_1)-f(\widetilde{\mathbf X}_2)\|_2^2],
\label{eq:population-alignment-risk}\\
\mathcal R_{\mathrm{FBDM}}(f,\phi)
&=\mathcal R_{\mathrm{FM}}(f,\phi)+\lambda\mathcal R_{\mathrm{align}}(f),
\qquad \lambda>0.
\label{eq:population-risk}
\end{align}
Table~\ref{tab:theory-notation} collects the analysis notation.

\begin{assumption}[$(\sigma,\delta)$-augmentation and encoder regularity]
\label{ass:augmentation-quality}
The kernel $\mathcal A$ samples uniformly from a fixed finite family
$\mathcal T=\{A_1,\ldots,A_{N_{\mathrm{aug}}}\}$ of measurable maps
$\mathcal X\to\mathcal X$, including the identity. For
$P_k=\operatorname{Law}(\mathbf X\mid Y=k)$, each class has a measurable core
$C_k^\circ\subseteq\operatorname{supp}(P_k)$ such that
\begin{equation}
P_k(C_k^\circ)\ge\sigma,\qquad
\sup_{\mathbf x,\mathbf x'\in C_k^\circ}
\min_{A,A'\in\mathcal T}\|A(\mathbf x)-A'(\mathbf x')\|_2\le\delta,
\label{eq:sigma-delta-augmentation}
\end{equation}
where $\sigma\in(0,1]$, $\delta\ge0$, and $f$ is $L_f$-Lipschitz.
\end{assumption}
Here $\sigma$ measures class-core coverage, while $\delta$ measures how closely
suitable views connect same-class images
\citep{huang2023towards,duan2025advssl,jiao2026distribution}.

\begin{assumption}[Uniform flow regularity and time coverage]
\label{ass:regular-flow}
The field $\bm v_\phi:\mathbb R^{d^\star}\times[0,1]\to\mathbb R^{d^\star}$
is measurable in $t$, with
\begin{equation}
\|\bm v_\phi(\mathbf z,t)-\bm v_\phi(\mathbf z',t)\|_2
\le\ell_\phi(t)\|\mathbf z-\mathbf z'\|_2,\qquad
\int_0^1\ell_\phi(t)\,\mathrm dt\le L_v<\infty.
\label{eq:flow-lipschitz}
\end{equation}
Here $\ell_\phi\ge0$ and $L_v$ is fixed over the admissible fields.
Also $\int_0^1\|\bm v_\phi(\mathbf0,t)\|_2\,\mathrm dt<\infty$.
Conditional paths are absolutely continuous, with
$\mathbb E\int_0^1\|\mathbf U_t\|_2^2\,\mathrm dt<\infty$ and
$\mathbb E\|\mathbf U_T\|_2^2<\infty$.
The fixed time density satisfies $q_T\ge q_{\min}>0$ almost everywhere.
\end{assumption}
For \eqref{eq:path-uniform}, $q_T(t)=\dot s(t)$. The schedules used in our
experiments satisfy $q_{\min}>0$; their explicit forms and parameters are
listed in Appendix~\ref{app:configs}, Table~\ref{tab:time-schedules}.

Let $\Phi_t$ solve
\begin{equation}
\frac{\mathrm d\Phi_t(\mathbf z)}{\mathrm dt}
=\bm v_\phi(\Phi_t(\mathbf z),t),\quad
\Phi_0(\mathbf z)=\mathbf z,\quad
\widehat{\mathbf Z}_t:=\Phi_t(\mathbf Z_0),\quad
\widehat p_t:=\operatorname{Law}(\widehat{\mathbf Z}_t),\quad t\in[0,1].
\label{eq:learned-ode}
\end{equation}
This learned terminal law need not equal the analytic endpoint law
$\operatorname{Law}(\mathbf R)=\nu_\varepsilon$. For probability laws $\mu$
and $\nu$ on $\mathbb R^{d^\star}$ with finite first moments, let
$\Pi(\mu,\nu)$ denote their set of couplings and define
\begin{equation}
W_1(\mu,\nu):=
\inf_{\gamma\in\Pi(\mu,\nu)}
\int\|\mathbf z-\mathbf r\|_2\,
\gamma(\mathrm d\mathbf z,\mathrm d\mathbf r).
\label{eq:theory-wasserstein}
\end{equation}
Although $W_1$ does not appear in the training objective,
Lemma~\ref{lem:fm-endpoint} shows that the FM risk controls this terminal
distributional mismatch:
\begin{equation}
W_1(\widehat p_1,\nu_\varepsilon)
\le e^{L_v}
\sqrt{\frac{d^\star}{q_{\min}}\mathcal R_{\mathrm{FM}}}.
\label{eq:main-fm-wasserstein}
\end{equation}

\paragraph{Semantic identifiability.}
The label-free bound \eqref{eq:main-fm-wasserstein} controls the overall
marginal transport but cannot by itself identify classes
\citep{jiao2026distribution}.
We isolate the required identifiability as
Assumption~\ref{ass:transport-identifiability};
Proposition~\ref{prop:leakage-gap} gives a geometric sufficient condition.
Let $\Gamma^\star\in\Pi(\widehat p_1,\nu_\varepsilon)$ attain the infimum in
\eqref{eq:theory-wasserstein}. We use the class posterior below to measure the
mass paired with each reference component under this analytical coupling.
Recall $\mathcal S_j^\varepsilon=\operatorname{supp}(Q_j^\varepsilon)$ and define
\begin{equation}
\pi_k(\widehat{\mathbf z}):=
P(Y=k\mid\widehat{\mathbf Z}_1=\widehat{\mathbf z}),
\qquad
q_{kj}^\star:=\int
\pi_k(\widehat{\mathbf z})
\mathbf 1\{\mathbf r\in\mathcal S_j^\varepsilon\}\,
\Gamma^\star(\mathrm d\widehat{\mathbf z},\mathrm d\mathbf r).
\label{eq:label-component-mass}
\end{equation}
Let $S_K$ denote the permutations of $[K]$. Choose the class--component
permutation retaining the most mass:
\begin{equation}
\tau^\star\in\arg\max_{\tau\in S_K}
\sum_{k=1}^K q_{k,\tau(k)}^\star.
\label{eq:optimal-class-component-permutation}
\end{equation}
Then
\begin{equation}
q_k^\star=q_{k,\tau^\star(k)}^\star,\qquad
\operatorname{Leak}(\Gamma^\star)=
\max_k\{p_k+w_{\tau^\star(k)}-2q_k^\star\}.
\label{eq:population-leakage}
\end{equation}
For each matched pair, the expression inside the maximum counts outgoing and
incoming mass.
This analytical coupling is distinct from the training-induced
representation--endpoint coupling $\Pi_f$.

\begin{assumption}[Transport identifiability]
\label{ass:transport-identifiability}
There is such an optimal coupling with
\begin{equation}
\operatorname{Leak}(\Gamma^\star)\le
C_{\mathrm{id}}W_1(\widehat p_1,\nu_\varepsilon),
\label{eq:leakage-control}
\end{equation}
where $C_{\mathrm{id}}>0$ is a constant.
\end{assumption}
Substituting the flow-matching guarantee
\eqref{eq:main-fm-wasserstein} into the identifiability condition
\eqref{eq:leakage-control} yields
\begin{equation}
\operatorname{Leak}(\Gamma^\star)
\le C_{\mathrm{id}}e^{L_v}
\sqrt{\frac{d^\star}{q_{\min}}\mathcal R_{\mathrm{FM}}}.
\label{eq:fm-controls-leakage}
\end{equation}

\subsection{A Population Misclassification Guarantee}
\label{subsec:classification-guarantee}
For each $k$, let $\mathbf X\sim P_k$ and
$A\sim\operatorname{Unif}(\mathcal T)$ be independent. Define the augmented
class means and nearest-centroid classifier by
\begin{equation}
\mathbf m_k=\mathbb E[f(A(\mathbf X))],
\qquad
G_f(\mathbf x)=\arg\min_{k\in[K]}\|f(\mathbf x)-\mathbf m_k\|_2^2,
\label{eq:population-classifier}
\end{equation}
with fixed tie-breaking, and $\operatorname{Err}(G_f)=P\{G_f(\mathbf X)\ne Y\}$.
Its scores $2\mathbf m_k^\top f(\mathbf x)-\|\mathbf m_k\|_2^2$ are affine in
the initial representation; classification does not run the ODE.

\begin{theorem}[Population FBDM loss controls classification error]
\label{thm:population-error}
Suppose Assumptions~\ref{ass:augmentation-quality}--\ref{ass:transport-identifiability}
hold. Fix a tolerance $\varepsilon_{\mathrm{al}}>0$ for distances between
representations of views of one image, such that
\begin{equation}
e^{-L_v}(\Delta_{\mathcal C}-2r_\varepsilon)
>4L_f\delta+8\varepsilon_{\mathrm{al}}+12(1-\sigma).
\label{eq:positive-augmentation-margin}
\end{equation}
There is an explicit threshold $r_0>0$, depending only on these fixed
parameters and specified in \eqref{eq:small-risk-constants}, such that
$\mathcal R_{\mathrm{FBDM}}(f,\phi)\le r_0$ implies
\begin{equation}
\operatorname{Err}(G_f)\le(1-\sigma)+
C_{\mathrm{err}}\sqrt{\mathcal R_{\mathrm{FBDM}}(f,\phi)},\qquad
C_{\mathrm{err}}=\frac{N_{\mathrm{aug}}}{\varepsilon_{\mathrm{al}}}
\sqrt{\frac{d^\star}{\lambda}}.
\label{eq:population-misclassification}
\end{equation}
\end{theorem}
The residual $1-\sigma$ accounts for images outside class cores.
Thus smaller full loss tightens the classification bound in the stated regime.
The guarantee concerns the initial normalized map $f$ with population centroids;
the experimental probes use backbone features.

%% file: sections/limitations_conclusion.tex
\section{Conclusion}

FBDM combines flexible reference geometry with non-adversarial velocity
regression, retaining only the backbone downstream. It yields
2.17--2.36$\times$ higher per-epoch training throughput than DM on the three
measured V100 benchmarks. Under stated conditions, our population analysis
bounds nearest-centroid error on the normalized encoder--projector
representation by the full FBDM loss and augmentation quality.

%% file: sections/appendix.tex
\input{appendices/geometry}
\input{appendices/flow}
\input{appendices/population}
\input{appendices/implementation}
\input{appendices/ablations}
\input{appendices/dm_diagnostics}

\clearpage
\input{appendices/related_work}

%% file: appendices/geometry.tex
\section{Reference Geometry}
\label{app:geometry}
This appendix supports the reference construction in
Section~\ref{subsec:reference}. Appendix~A.1 establishes the angular benchmark
provided by ETFs; Appendix~A.2 verifies the exact simplex ETF encoded by
$\mathbf G_\Delta$ and gives its spectral construction; Appendix~A.3 justifies
the lower-dimensional approximation.

\subsection{ETF Angular Benchmark}
\label{app:proof-welch}
We first ask how far apart $K'$ unit centers can be placed in
$\mathbb R^{d^\star}$. The following definition and proposition formalize the
ETF benchmark invoked in Section~\ref{subsec:reference}.

\begin{definition}[Equiangular tight frame]
\label{def:etf}
Let $K'\ge d^\star\ge2$. A collection
$\mathcal C=\{\mathbf c_k\}_{k=1}^{K'}\subset\Sphere^{d^\star-1}$
is an equiangular tight frame (ETF) if
\begin{equation}
\sum_{k=1}^{K'}\mathbf c_k\mathbf c_k^\top
=\frac{K'}{d^\star}\mathbf I_{d^\star},
\qquad
|\mathbf c_k^\top\mathbf c_\ell|=\alpha
\quad(k\ne\ell)
\label{eq:etf-inner-product}
\end{equation}
for a constant $\alpha\ge0$. The first identity gives every unit direction
the same total squared projection onto the centers; the second makes all
pairwise inner products equal in magnitude.
\end{definition}

Here angular separation is measured between unoriented lines: changing a
center's sign leaves the angle unchanged. The classification analysis uses
the Euclidean center distances in \eqref{eq:theory-reference-geometry}.
For $K'\ge d^\star$, the coherence
$\mu$ measures the most aligned pair under this convention:

\begin{equation}
\mu(\mathcal C):=\max_{k\ne\ell}|\mathbf c_k^\top\mathbf c_\ell|,
\qquad
\theta_{\min}^{\pm}(\mathcal C):=\arccos(\mu(\mathcal C)).
\label{eq:coherence-angle}
\end{equation}

\begin{proposition}[ETF angular benchmark]
\label{prop:welch}
Let $K'\ge d^\star\ge2$. Every collection
$\mathcal C=\{\mathbf c_k\}_{k=1}^{K'}\subset\Sphere^{d^\star-1}$ satisfies
\begin{equation}
\mu(\mathcal C)\ge
\mu_{\mathrm W}:=
\sqrt{\frac{K'-d^\star}{d^\star(K'-1)}},
\qquad
\theta_{\min}^{\pm}(\mathcal C)\le\arccos(\mu_{\mathrm W}).
\label{eq:welch}
\end{equation}
Equality holds if and only if $\mathcal C$ is an ETF
in the sense of Definition~\ref{def:etf}.
\end{proposition}
Equivalently, every placement of $K'$ centers contains at least one pair whose
sign-invariant angle is at most $\arccos(\mu_{\mathrm W})$; when an ETF exists,
it makes this worst-pair angle as large as possible \citep{welch1974lower}.

\begin{proof}[Proof of Proposition~\ref{prop:welch}]
Assume $K'\ge d^\star\ge2$. Let
$\mathbf C\in\mathbb R^{K'\times d^\star}$ have $k$th row
$\mathbf c_k^\top$, and let
$\mathbf G=\mathbf C\mathbf C^\top$. With
$\mu=\max_{k\ne\ell}|\mathbf c_k^\top\mathbf c_\ell|$,
\begin{equation}
\|\mathbf G\|_F^2
=K'+\sum_{k\ne\ell}
|\mathbf c_k^\top\mathbf c_\ell|^2
\le K'+K'(K'-1)\mu^2.
\end{equation}
Thus, a lower bound on $\mu$ follows from a lower bound on
$\|\mathbf G\|_F^2$. Let $r=\operatorname{rank}(\mathbf G)\le d^\star$ and
let $\lambda_1(\mathbf G),\ldots,\lambda_r(\mathbf G)>0$ be its nonzero
eigenvalues. Since $\mathbf G$ is positive semidefinite and the centers have
unit norm,
$\operatorname{tr}(\mathbf G)=\sum_{i=1}^r\lambda_i(\mathbf G)=K'$.
Cauchy--Schwarz therefore gives
\begin{equation}
\begin{aligned}
\|\mathbf G\|_F^2
&=\sum_{i=1}^r\lambda_i(\mathbf G)^2
\ge\frac{1}{r}\left\{\sum_{i=1}^r\lambda_i(\mathbf G)\right\}^2\\
&=\frac{\operatorname{tr}(\mathbf G)^2}{r}
\ge\frac{\operatorname{tr}(\mathbf G)^2}{d^\star}
=\frac{K'^2}{d^\star}.
\end{aligned}
\end{equation}
The first inequality is Cauchy--Schwarz, and the second uses
$r\le d^\star$.
Substituting this bound into the preceding inequality yields
\begin{equation}
\mu^2\ge
\frac{K'-d^\star}{d^\star(K'-1)},
\end{equation}
which is \eqref{eq:welch}. Equality requires equality in both preceding
steps. Equality in the first bound means that all off-diagonal inner products
have the same magnitude, which is equiangularity. In the eigenvalue bound,
equality in Cauchy--Schwarz requires
$\lambda_1(\mathbf G)=\cdots=\lambda_r(\mathbf G)$, while equality in the
step $r\le d^\star$ requires $r=d^\star$. Together these conditions are
equivalent to
$\mathbf C^\top\mathbf C=(K'/d^\star)\mathbf I_{d^\star}$, which is tightness.
The converse follows by substituting the ETF identities into the two bounds.
Because $\arccos$ is decreasing on $[0,1]$,
$\theta_{\min}^{\pm}(\mathcal C)=\arccos(\mu(\mathcal C))$ gives the
angular inequality in \eqref{eq:welch}, with equality under exactly the same
ETF conditions.
\end{proof}

\subsection{Simplex Gram Matrix and Spectral Construction}
\label{app:spectral-construction}
The construction in Section~\ref{subsec:reference} starts from the following
goal: realize $K'$ unit vectors in $\mathbb R^{K'-1}$ that are simultaneously
equiangular and tight. Rather than choosing their coordinates directly, we
first encode the desired geometry in the ideal Gram matrix
$\mathbf G_\Delta$ from \eqref{eq:simplex-gram}, whose entries are
\[
(\mathbf G_\Delta)_{k\ell}=
\begin{cases}
1,&k=\ell,\\
-1/(K'-1),&k\ne\ell.
\end{cases}
\]
The diagonal directly prescribes unit row norms, while the off-diagonal
entries prescribe the common inner product $-1/(K'-1)$.

The matrix $\mathbf G_\Delta$ is positive semidefinite with
$\operatorname{rank}(\mathbf G_\Delta)=K'-1$. Its $K'-1$ positive
eigenvalues are all $K'/(K'-1)$, and its remaining eigenvalue is zero.
Equivalently, its full spectral decomposition is
\[
\mathbf G_\Delta=\mathbf U\bm\Lambda\mathbf U^\top,
\qquad
\mathbf U\in\mathbb R^{K'\times K'},
\quad
\bm\Lambda=\operatorname{diag}\!\left(
\frac{K'}{K'-1}\mathbf I_{K'-1},0\right)
\in\mathbb R^{K'\times K'}.
\]
Thus the Gram realization theorem guarantees $K'$ vectors in
$\mathbb R^{K'-1}$ with this Gram matrix.

For an explicit realization, let
$\mathbf U_+\in\mathbb R^{K'\times(K'-1)}$ and
$\bm\Lambda_+\in\mathbb R^{(K'-1)\times(K'-1)}$ denote the eigenvectors and
diagonal block associated with the positive eigenvalues. The subscript $+$
always refers to this positive-eigenvalue block. Since the omitted spectral
term is $0\,\mathbf u_0\mathbf u_0^\top=\mathbf0$, the thin factorization
remains exact:
\[
\mathbf G_\Delta=\mathbf U_+\bm\Lambda_+\mathbf U_+^\top,
\qquad
\bm\Lambda_+=\frac{K'}{K'-1}\mathbf I_{K'-1}.
\]
Set
$\mathbf C_\Delta=\mathbf U_+\bm\Lambda_+^{1/2}
\in\mathbb R^{K'\times(K'-1)}$. Then
\[
\mathbf C_\Delta\mathbf C_\Delta^\top=\mathbf G_\Delta,
\qquad
\mathbf C_\Delta^\top\mathbf C_\Delta
=\frac{K'}{K'-1}\mathbf I_{K'-1}.
\]
The first identity shows that the rows of $\mathbf C_\Delta$ have unit norm
and pairwise inner product $-1/(K'-1)$; the second establishes tightness.
Therefore these rows form the exact simplex ETF of $K'$ vectors in
$\mathbb R^{K'-1}$ from Definition~\ref{def:etf}.

Only now do we reduce the dimension. For $2\le d^\star<K'-1$, retain
$d^\star$ directions from the positive eigenspace and their eigenvalues:
\[
\mathbf U_{d^\star}\in\mathbb R^{K'\times d^\star},
\qquad
\bm\Lambda_{d^\star}=\frac{K'}{K'-1}\mathbf I_{d^\star}
\in\mathbb R^{d^\star\times d^\star}.
\]
Equation~\eqref{eq:projected-coordinates} forms $\mathbf Y$ from these
matrices. Lemma~\ref{lem:spectral} shows that
$\mathbf Y\mathbf Y^\top$ is an optimal rank-$d^\star$ approximation of the
ideal Gram geometry. At $d^\star=K'-1$, retaining the entire positive
eigenspace recovers the exact simplex ETF.

As in Section~\ref{subsec:reference}, select the retained eigenspace so that
its rows are nonzero and its normalized directions are distinct.
Algorithm~\ref{alg:reference-centers} summarizes the construction.

\begin{algorithm}[H]
\caption{ETF-inspired simplex-spectral center construction.}
\label{alg:reference-centers}
\begin{algorithmic}[1]
\Require Number of centers $K'$ and matching dimension $2\le d^\star\le K'-1$
\Ensure Unit centers $\mathcal C=\{\mathbf c_k\}_{k=1}^{K'}\subset\Sphere^{d^\star-1}$
\State Form $\mathbf G_\Delta\gets\frac{K'}{K'-1}
\left(\mathbf I_{K'}-\frac{\mathbf1\mathbf1^\top}{K'}\right)
\in\mathbb R^{K'\times K'}$
\State Compute its nonzero eigenspace, with eigenvalue $K'/(K'-1)$
\State Select $d^\star$ orthonormal vectors from this eigenspace as
  $\mathbf U_{d^\star}\in\mathbb R^{K'\times d^\star}$
\State Set $\bm\Lambda_{d^\star}\gets\frac{K'}{K'-1}\mathbf I_{d^\star}$
and $\mathbf Y\gets\mathbf U_{d^\star}\bm\Lambda_{d^\star}^{1/2}$,
with rows $\mathbf y_k^\top$
\For{$k=1,\ldots,K'$}
  \State $\mathbf c_k\gets\mathbf y_k/\|\mathbf y_k\|_2$
\EndFor
\State \textbf{return} $\mathcal C$
\end{algorithmic}
\end{algorithm}

\subsection{Spectral Approximation}
We finally justify the approximation claim following
\eqref{eq:projected-coordinates}.

Using the eigenpairs defined above, the projected coordinate matrix is
\[
\mathbf Y=\mathbf U_{d^\star}\bm\Lambda_{d^\star}^{1/2}
\in\mathbb R^{K'\times d^\star},
\qquad \mathbf Y_{k,:}=\mathbf y_k^\top.
\]
The following result concerns these coordinates before row normalization.
\begin{lemma}[Optimal preservation of pairwise simplex geometry]
\label{lem:spectral}
Let $K'>d^\star\ge2$, and let $\widetilde{\mathbf y}_k^\top$ denote the
$k$th row of $\widetilde{\mathbf Y}$. Among all configurations of $K'$ points
in $\mathbb R^{d^\star}$, the spectral coordinates in
\eqref{eq:projected-coordinates}
minimize the total squared distortion of the ideal simplex pairwise inner
products:
\begin{gather}
\mathbf Y\in\arg\min_{\widetilde{\mathbf Y}\in
\mathbb R^{K'\times d^\star}}
\left\|\mathbf G_\Delta-
\widetilde{\mathbf Y}\widetilde{\mathbf Y}^{\top}\right\|_F^2,
\label{eq:spectral-optimum}\\
\left\|\mathbf G_\Delta-
\widetilde{\mathbf Y}\widetilde{\mathbf Y}^{\top}\right\|_F^2
=\sum_{k,\ell=1}^{K'}
\left\{(\mathbf G_\Delta)_{k\ell}
-\widetilde{\mathbf y}_k^\top\widetilde{\mathbf y}_\ell\right\}^2.
\nonumber
\end{gather}
\end{lemma}

The first line is the classical Eckart--Young low-rank approximation result
\citep{eckart1936approximation}; the second line makes its geometry explicit.
Dividing by $K'^2$ shows that spectral projection minimizes the mean squared
error over all corresponding pairwise inner products before row normalization.
Row normalization then places the projected coordinates on
$\Sphere^{d^\star-1}$ without changing their directions or pairwise angles.
We therefore call \eqref{eq:projected-coordinates} an
\emph{ETF-inspired simplex-spectral reference}, not an exact ETF.

%% file: appendices/flow.tex
\section{Flow-Matching Derivations}
\label{app:flow}
\subsection{Spherical Path and Target Velocity}
\label{app:derive-slerp}
We derive the great-circle interpolation and its derivative with respect to
path progress $s$; the clocked training target then follows by the chain rule
in \eqref{eq:clocked-target}.
Fix non-antipodal $\mathbf z,\mathbf r\in\Sphere^{d^\star-1}$ and let
$\omega=\arccos(\mathbf z^\top\mathbf r)$. Since $\|\mathbf z\|_2=1$,
$\mathbf z^\top\mathbf r$ is the scalar projection of $\mathbf r$ onto
$\mathbf z$. Since $\mathbf z$ is normal to the sphere at $\mathbf z$, the
residual lies in that point's tangent space. Thus
\begin{equation}
\mathbf r
=\underbrace{(\mathbf z^\top\mathbf r)\mathbf z}_{\text{normal component}}
+\underbrace{\left[\mathbf r-(\mathbf z^\top\mathbf r)\mathbf z\right]}_{
\text{tangent component}}
=\cos(\omega)\mathbf z+\left[\mathbf r-\cos(\omega)\mathbf z\right].
\end{equation}
The two components are orthogonal. For $0<\omega<\pi$, the second has norm
$\sin(\omega)$, so normalizing it gives
\begin{equation}
\mathbf u_\perp
:=\frac{\mathbf r-\cos(\omega)\mathbf z}{\sin(\omega)}.
\end{equation}
Then $\{\mathbf z,\mathbf u_\perp\}$ is an orthonormal basis of the plane
spanned by $\mathbf z$ and $\mathbf r$. The intersection of this plane with
the unit sphere is the great circle containing both endpoints. Its arc from
$\mathbf z$ to $\mathbf r$ is therefore
\begin{equation}
\bm\gamma_{\mathbf z,\mathbf r}(s)
=\cos(s\omega)\mathbf z+\sin(s\omega)\mathbf u_\perp,
\qquad s\in[0,1].
\end{equation}
Substituting $\mathbf u_\perp$ and using
$\sin((1-s)\omega)=\sin(\omega)\cos(s\omega)
-\cos(\omega)\sin(s\omega)$ gives
\begin{align}
\bm\gamma_{\mathbf z,\mathbf r}(s)
&=\left[\cos(s\omega)
-\frac{\cos(\omega)\sin(s\omega)}{\sin(\omega)}\right]\mathbf z
+\frac{\sin(s\omega)}{\sin(\omega)}\mathbf r\\
&=\frac{\sin((1-s)\omega)}{\sin(\omega)}\mathbf z
+\frac{\sin(s\omega)}{\sin(\omega)}\mathbf r,
\end{align}
which is \eqref{eq:slerp}. Moreover,
\begin{equation}
\partial_s\bm\gamma_{\mathbf z,\mathbf r}(s)
=\omega\left[-\sin(s\omega)\mathbf z
+\cos(s\omega)\mathbf u_\perp\right],
\end{equation}
Equivalently, the expanded derivative is
\begin{equation}
\partial_s\bm\gamma_{\mathbf z,\mathbf r}(s)
=-\frac{\omega\cos((1-s)\omega)}{\sin\omega}\mathbf z
+\frac{\omega\cos(s\omega)}{\sin\omega}\mathbf r.
\label{eq:slerp-derivative}
\end{equation}
Orthogonality yields
\begin{equation}
\|\bm\gamma_{\mathbf z,\mathbf r}(s)\|_2=1,
\qquad
\bm\gamma_{\mathbf z,\mathbf r}(s)^\top
\partial_s\bm\gamma_{\mathbf z,\mathbf r}(s)=0,
\qquad
\|\partial_s\bm\gamma_{\mathbf z,\mathbf r}(s)\|_2=\omega.
\end{equation}
The endpoints are
$\bm\gamma_{\mathbf z,\mathbf r}(0)=\mathbf z$ and
$\bm\gamma_{\mathbf z,\mathbf r}(1)=\mathbf r$, and the curve length is
$\int_0^1\omega\,\mathrm ds=\omega$, equal to the spherical distance
$\arccos(\mathbf z^\top\mathbf r)$. Hence it is the unique shortest geodesic
for $0<\omega<\pi$. At $\omega=0$ the same formula gives the constant path by
continuity; the antipodal case $\omega=\pi$, where the shortest geodesic is not
unique, is excluded in Section~\ref{sec:method}.

%% file: appendices/population.tex
\section{Theoretical Analysis}
\label{app:population}

\subsection{Key Notation and Proof Sketch}
\label{app:theory-notation}
\begin{table}[H]
\centering
\caption{Key notation for the flow and classification analysis. The index $k$
denotes a semantic class, and $j$ denotes a reference component. All representation
vectors lie in $\mathbb R^{d^\star}$.}
\label{tab:theory-notation}
\small
\renewcommand{\arraystretch}{1.12}
\begin{tabular}{@{}p{0.14\textwidth}p{0.65\textwidth}p{0.15\textwidth}@{}}
\toprule
Symbol & Meaning / definition & First defined\\
\midrule
$\mathbf Z_0$ & Encoder representation $f(\widetilde{\mathbf X})$ before ODE evolution;
used for downstream classification. & Section~\ref{subsec:flow}\\
$\mathbf R$ & Assigned reference target; the prescribed endpoint of the analytic path.
& Section~\ref{subsec:flow}\\
$\mathbf Z_t$ & State on the prescribed analytic path, with $\mathbf Z_1=\mathbf R$.
& \eqref{eq:clocked-law}\\
$\mathbf U_t$ & Conditional target velocity $\mathrm d\mathbf Z_t/\mathrm dt$.
& \eqref{eq:clocked-target}\\
$\Phi_t$ & Learned ODE map satisfying
$\mathrm d\Phi_t(\mathbf z)/\mathrm dt=\bm v_\phi(\Phi_t(\mathbf z),t)$ and
$\Phi_0(\mathbf z)=\mathbf z$. & \eqref{eq:learned-ode}\\
$\widehat{\mathbf Z}_t$ & Learned ODE state $\Phi_t(\mathbf Z_0)$;
$\widehat{\mathbf Z}_1$ is its terminal representation, not the prescribed target.
& \eqref{eq:learned-ode}\\
\midrule
$\mathbf m_k$ & Class mean before ODE evolution:
$\mathbb E[\mathbf Z_0\mid Y=k]$. & \eqref{eq:population-classifier}\\
$\widehat{\mathbf m}_k$ & Class mean after learned ODE evolution:
$\mathbb E[\widehat{\mathbf Z}_1\mid Y=k]$; generally not $\Phi_1(\mathbf m_k)$.
& \eqref{eq:terminal-reference-means}\\
$\mathbf c_j$ & Geometric center of reference component $j$.
& Section~\ref{subsec:reference}\\
$\mathbf b_j$ & Reference component mean
$\int\mathbf r\,Q_j^\varepsilon(\mathrm d\mathbf r)$; generally not $\mathbf c_j$.
& \eqref{eq:terminal-reference-means}\\
\midrule
$\nu_\varepsilon$ & Fixed reference mixture $\sum_jw_jQ_j^\varepsilon$.
& \eqref{eq:reference-law}; Section~\ref{subsec:population-setup}\\
$\widehat p_1$ & Learned terminal marginal
$\operatorname{Law}(\widehat{\mathbf Z}_1)$. & \eqref{eq:learned-ode}\\
$\Pi_f$ & Training representation--target coupling
$\operatorname{Law}(\mathbf Z_0,\mathbf R)$;
global transport optimality is not required. & Section~\ref{subsec:flow}\\
$\Gamma^\star$ & Analytical optimal $W_1$ coupling of
$\widehat p_1$ and $\nu_\varepsilon$:
$\operatorname{Law}(\widehat{\mathbf Z}_1,\mathbf R^\star)$.
& Section~\ref{subsec:population-setup}\\
\bottomrule
\end{tabular}
\end{table}

\paragraph{Proof sketch.}
The proof of Theorem~\ref{thm:population-error} distinguishes the class means
before and after the learned ODE, $\mathbf m_k$ and $\widehat{\mathbf m}_k$,
from the reference component mean $\mathbf b_j$ and its geometric center
$\mathbf c_j$. Table~\ref{tab:theory-notation} records their definitions.
Here $h_{\mathrm{aug}}=L_f\delta+2\varepsilon_{\mathrm{al}}$ bounds the
representation diameter of a stable class core, while
$b_{\mathrm{aug}}=1-\sigma+\mathcal U/p_{\min}$ bounds the class-conditional mass outside
that core. The unstable-view mass $\mathcal U$ is defined in
\eqref{eq:stable-view-set}. The argument proceeds in four steps.
\begin{enumerate}
\item \textbf{Match terminal class means to reference means.}
Lemma~\ref{lem:fm-endpoint} bounds
$W_1(\widehat p_1,\nu_\varepsilon)$ by the square root of the FM risk.
Under the leakage condition in
Assumption~\ref{ass:transport-identifiability},
Lemma~\ref{lem:leakage-means} converts this overall distributional bound into
a bound for each terminal class mean. Using
$\mathcal R_{\mathrm{FM}}\le\mathcal R_{\mathrm{FBDM}}$ gives
\[
\|\widehat{\mathbf m}_k-\mathbf b_{\tau^\star(k)}\|_2
\le\frac{(1+2C_{\mathrm{id}})e^{L_v}}{p_{\min}}
\sqrt{\frac{d^\star}{q_{\min}}\mathcal R_{\mathrm{FBDM}}}.
\]
Thus a small FBDM loss makes each ODE-terminal class mean close to the mean
of its corresponding reference component.

\item \textbf{Obtain separation at the ODE endpoint.}
The reference component means themselves are separated: each lies within
$r_\varepsilon$ of its geometric center, so
\[
\|\mathbf b_i-\mathbf b_j\|_2
\ge\Delta_{\mathcal C}-2r_\varepsilon>0,\qquad i\ne j.
\]
Since each terminal class mean is close to a different reference component
mean, the triangle inequality transfers this separation to
$\widehat{\mathbf m}_k$ and $\widehat{\mathbf m}_\ell$ when the loss-dependent
errors in Step~1 are sufficiently small.

\item \textbf{Recover separation of the encoder class means.}
Downstream classification uses the initial encoder representations
$\mathbf Z_0$, so the class means we actually need to separate are
$\mathbf m_k$ and $\mathbf m_\ell$, not their terminal counterparts.
Lemma~\ref{lem:fm-endpoint} establishes the forward Lipschitz bound
\[
\|\Phi_1(\mathbf m_k)-\Phi_1(\mathbf m_\ell)\|_2
\le e^{L_v}\|\mathbf m_k-\mathbf m_\ell\|_2.
\]
The ODE can enlarge distances by at most $e^{L_v}$, so separated evolved
means require separated initial means. Here $\Phi_1(\mathbf m_k)$ evolves
the initial class mean, whereas $\widehat{\mathbf m}_k$ averages the
individually evolved representations. Equation~\eqref{eq:transported-class-mean}
in Section~\ref{app:population-proof} gives
\[
\|\Phi_1(\mathbf m_k)-\widehat{\mathbf m}_k\|_2
\le e^{L_v}(h_{\mathrm{aug}}+4b_{\mathrm{aug}}).
\]
Subtracting these discrepancies from the terminal separation and dividing
by $e^{L_v}$ yields the initial class-mean distance lower bound
\eqref{eq:class-mean-separation}.

\item \textbf{Turn separation and concentration into classification.}
Lemma~\ref{lem:augmentation-core-distances} gives the within-class
concentration bound \eqref{eq:augmented-class-spread}:
\[
\mathbb E_{\substack{\mathbf X\sim P_k,\,
A\sim\operatorname{Unif}(\mathcal T)}}
[\|f(A(\mathbf X))-\mathbf m_k\|_2]
\le h_{\mathrm{aug}}+4b_{\mathrm{aug}}.
\]
Combining Steps~1--3 gives the initial class-mean distance bound
\eqref{eq:class-mean-separation}. Write
$\Delta_f:=\min_{k\ne\ell}\|\mathbf m_k-\mathbf m_\ell\|_2$.
Substituting the loss bound from Lemma~\ref{lem:fm-endpoint} gives
\[
\begin{aligned}
\min_{k\ne\ell}\|\mathbf m_k-\mathbf m_\ell\|_2
&\ge e^{-L_v}(\Delta_{\mathcal C}-2r_\varepsilon)
-2h_{\mathrm{aug}}-8b_{\mathrm{aug}}\\
&\quad-\frac{2(1+2C_{\mathrm{id}})}{p_{\min}}
\sqrt{\frac{d^\star}{q_{\min}}\mathcal R_{\mathrm{FBDM}}}.
\end{aligned}
\]
Thus the encoder representations already have within-class concentration
and between-class separation at the ODE's initial position. More precisely,
\eqref{eq:stable-core-radius} and the theorem's small-loss condition
\eqref{eq:classification-margin-from-loss} give
\[
\|f(\mathbf x)-\mathbf m_k\|_2
\le h_{\mathrm{aug}}+2b_{\mathrm{aug}}<\frac{\Delta_f}{2},
\qquad \mathbf x\in C_k^\circ\cap\mathcal V.
\]
The triangle inequality therefore makes each stable class-core
representation closer to its own mean than to any other class mean.
Only observations outside these cores or with unstable views can be
misclassified. Lemma~\ref{lem:alignment-stability} bounds their mass,
giving Theorem~\ref{thm:population-error}:
\[
\operatorname{Err}(G_f)
\le(1-\sigma)+C_{\mathrm{err}}\sqrt{\mathcal R_{\mathrm{FBDM}}}.
\]
The FM risk ensures sufficient between-class separation, while alignment
controls within-class concentration and the unstable mass.
\end{enumerate}
Section~\ref{app:population-laws} fixes the probability laws,
and Section~\ref{app:population-proof} gives the full proof.
Section~\ref{app:leakage-sufficient} gives a sufficient condition for leakage control.

\subsection{Probability Laws}
\label{app:population-laws}
Let $(\mathbf X,Y)\sim P$, with $Y\in[K]$ and $K\ge2$. Write
\begin{equation}
p_k:=P(Y=k)>0,\qquad p_{\min}:=\min_{k\in[K]}p_k>0.
\label{eq:class-probabilities}
\end{equation}
For measurable $E\subseteq\mathcal X$, define the marginal law
$P_X(E):=P(\mathbf X\in E)$ and let
$P_k:=\operatorname{Law}(\mathbf X\mid Y=k)$. The augmentation kernel
$\mathcal A$ is fixed independently of the networks. Given $\mathbf X$,
draw $\widetilde{\mathbf X}_1,\widetilde{\mathbf X}_2$ independently from
$\mathcal A(\,\cdot\mid\mathbf X)$. For measurable
$E,E_1,E_2\subseteq\mathcal X$, define the augmented-image and positive-pair
laws by
\begin{align}
P_{\mathcal A}(E)
&:=\int\mathcal A(E\mid\mathbf x)\,P_X(\mathrm d\mathbf x),
\label{eq:fixed-augmented-law}\\
\Pi_{\mathcal A}^{+}(E_1\times E_2)
&:=\int\mathcal A(E_1\mid\mathbf x)\mathcal A(E_2\mid\mathbf x)
\,P_X(\mathrm d\mathbf x).
\label{eq:fixed-positive-law}
\end{align}
Define $P_{\mathcal A,k}$ by replacing $P_X$ with $P_k$ in
\eqref{eq:fixed-augmented-law}. The positive-pair law retains the
shared-image dependence and is generally not
$P_{\mathcal A}\otimes P_{\mathcal A}$.

For later use, the augmentation distance is
\begin{equation}
d_{\mathcal T}(\mathbf x,\mathbf x')
:=\min_{A,A'\in\mathcal T}
\|A(\mathbf x)-A'(\mathbf x')\|_2.
\label{eq:augmentation-distance}
\end{equation}

\subsection{Proof of Theorem~\ref{thm:population-error}}
\label{app:population-proof}

We combine the stable-core argument of
\citet{huang2023towards,jiao2026distribution} with terminal flow matching
and a pullback of class-centre separation. The classifier centres are
averages of augmented representations, whereas the classified inputs are
original images. All constants use the fixed bounds in
Section~\ref{sec:theory}.

\subsubsection{Alignment and Within-Class Concentration}
\begin{lemma}[Alignment controls unstable-view probability]
\label{lem:alignment-stability}
For the tolerance $\varepsilon_{\mathrm{al}}$ in
Theorem~\ref{thm:population-error}, define
\begin{align}
\mathcal V
&:=\left\{\mathbf x\in\operatorname{supp}(P_X):
\max_{A,A'\in\mathcal T}
\|f(A(\mathbf x))-f(A'(\mathbf x))\|_2
\le\varepsilon_{\mathrm{al}}\right\},\nonumber\\
\mathcal U&:=P_X(\mathcal V^c).
\label{eq:stable-view-set}
\end{align}
Under Assumption~\ref{ass:augmentation-quality},
\begin{equation}
\mathcal U
\le\frac{N_{\mathrm{aug}}}{\varepsilon_{\mathrm{al}}}
\sqrt{d^\star\mathcal R_{\mathrm{align}}}
\le C_{\mathrm{err}}\sqrt{\mathcal R_{\mathrm{FBDM}}}.
\label{eq:alignment-stability-bound}
\end{equation}
\end{lemma}
\begin{proof}
For each fixed $\mathbf x$, the largest pairwise distance is bounded by
the square root of the sum of all squared pairwise distances.
Uniform sampling from $\mathcal T$ gives
\begin{align}
\max_{A,A'\in\mathcal T}
\|f(A(\mathbf x))-f(A'(\mathbf x))\|_2
&\le\left[\sum_{i,j=1}^{N_{\mathrm{aug}}}
\|f(A_i(\mathbf x))-f(A_j(\mathbf x))\|_2^2\right]^{1/2}\nonumber\\
&=N_{\mathrm{aug}}\left[
\mathbb E_{A,A'}\|f(A(\mathbf x))-f(A'(\mathbf x))\|_2^2
\right]^{1/2}.
\end{align}
Markov's inequality first gives
\begin{align}
\mathcal U
&=P_X\!\left(\left\{\mathbf x:
\max_{A,A'\in\mathcal T}
\|f(A(\mathbf x))-f(A'(\mathbf x))\|_2
>\varepsilon_{\mathrm{al}}\right\}\right)\nonumber\\
&\le\frac{1}{\varepsilon_{\mathrm{al}}}
\mathbb E_{\mathbf X\sim P_X}\!\left[
\max_{A,A'\in\mathcal T}
\|f(A(\mathbf X))-f(A'(\mathbf X))\|_2\right].
\label{eq:alignment-markov-step}
\end{align}
Using the pointwise bound above and then Cauchy--Schwarz,
\begin{align}
&\mathbb E_{\mathbf X\sim P_X}\!\left[
\max_{A,A'\in\mathcal T}
\|f(A(\mathbf X))-f(A'(\mathbf X))\|_2\right]\nonumber\\
&\le N_{\mathrm{aug}}\,
\mathbb E_{\mathbf X\sim P_X}\!\left[
\left\{\mathbb E_{A,A'}
\|f(A(\mathbf X))-f(A'(\mathbf X))\|_2^2\right\}^{1/2}
\right]\nonumber\\
&\le N_{\mathrm{aug}}\left[
\mathbb E_{\mathbf X,A,A'}
\|f(A(\mathbf X))-f(A'(\mathbf X))\|_2^2
\right]^{1/2}\nonumber\\
&=N_{\mathrm{aug}}
\sqrt{d^\star\mathcal R_{\mathrm{align}}}.
\label{eq:alignment-cauchy-step}
\end{align}
Combining \eqref{eq:alignment-markov-step} and
\eqref{eq:alignment-cauchy-step}, and using
$\lambda\mathcal R_{\mathrm{align}}
\le\mathcal R_{\mathrm{FBDM}}$, yields
\begin{equation}
\mathcal U
\le\frac{N_{\mathrm{aug}}}{\varepsilon_{\mathrm{al}}}
\sqrt{d^\star\mathcal R_{\mathrm{align}}}
\le\frac{N_{\mathrm{aug}}}{\varepsilon_{\mathrm{al}}}
\sqrt{\frac{d^\star}{\lambda}}
\sqrt{\mathcal R_{\mathrm{FBDM}}}
=C_{\mathrm{err}}\sqrt{\mathcal R_{\mathrm{FBDM}}}.
\end{equation}
Thus $N_{\mathrm{aug}}$ is precisely the cost of converting the
average pairwise alignment error over the finite augmentation family into
uniform control over all augmentation pairs.
\end{proof}

The next lemma turns view-level alignment into class-level concentration.
The $(\sigma,\delta)$-augmentation assumption connects same-class images
through nearby views, while Lemma~\ref{lem:alignment-stability} controls the
images whose views are not aligned. Together with Lipschitz continuity, these
facts place most class-$k$ representations near their class mean. This
within-class radius is used later to convert the separation pulled back by the
learned flow into a nearest-centroid guarantee.

\begin{lemma}[Augmentation cores control within-class distances]
\label{lem:augmentation-core-distances}
Recall from \eqref{eq:population-classifier} that
$\mathbf m_k=\mathbb E[f(A(\mathbf X))]$ for independent
$\mathbf X\sim P_k$ and $A\sim\operatorname{Unif}(\mathcal T)$, and set
\begin{equation}
h_{\mathrm{aug}}:=L_f\delta+2\varepsilon_{\mathrm{al}},\qquad
b_{\mathrm{aug}}:=1-\sigma+\frac{\mathcal U}{p_{\min}}.
\label{eq:core-distance-constants}
\end{equation}
Under Assumption~\ref{ass:augmentation-quality}, for every $k$,
\begin{equation}
\mathbb E_{\substack{\mathbf X\sim P_k,\,
A\sim\operatorname{Unif}(\mathcal T)}}
[\|f(A(\mathbf X))-\mathbf m_k\|_2]
\le h_{\mathrm{aug}}+4b_{\mathrm{aug}}.
\label{eq:augmented-class-spread}
\end{equation}
Moreover, every original image
$\mathbf x\in C_k^\circ\cap\mathcal V$ satisfies
\begin{equation}
\|f(\mathbf x)-\mathbf m_k\|_2
\le h_{\mathrm{aug}}+2b_{\mathrm{aug}}.
\label{eq:stable-core-radius}
\end{equation}
The ideal regime is transparent from these bounds. As
$\sigma\to1$, $\delta\to0$, $\varepsilon_{\mathrm{al}}\to0$, and
$\mathcal U\to0$, both $h_{\mathrm{aug}}$ and
$b_{\mathrm{aug}}$ vanish, so the within-class radius contracts to zero.
In the exact limit, \eqref{eq:augmented-class-spread} gives
$\mathbf Z_0=\mathbf m_k$ almost surely conditional on $Y=k$; hence the
class-conditional representation has zero variance. Near this limit, the two
bounds quantify how imperfect augmentation coverage, augmented distance, and
view alignment broaden each class around its mean.
\end{lemma}
\begin{proof}
The union bound, $P_k(C_k^\circ)\ge\sigma$, and $p_k\ge p_{\min}$ give
\begin{align}
P_k\!\left(\left\{C_k^\circ\cap\mathcal V\right\}^c\right)
&\le P_k\!\left(\left\{C_k^\circ\right\}^c\right)
+P_k\!\left(\mathcal V^c\right)\nonumber\\
&\le1-\sigma+
\frac{P(\mathbf X\in\mathcal V^c,Y=k)}{p_k}
\nonumber\\
&\le1-\sigma+\frac{\mathcal U}{p_k}
\le1-\sigma+\frac{\mathcal U}{p_{\min}}
=b_{\mathrm{aug}}.
\label{eq:core-uncovered-mass}
\end{align}
For $\mathbf x,\mathbf x'\in
C_k^\circ\cap\mathcal V$, choose
$A_*,A_*'\in\mathcal T$
attaining the minimum in \eqref{eq:augmentation-distance}. For any
$A,A'\in\mathcal T$, Lipschitz continuity and view stability give
\begin{align}
\|f(A(\mathbf x))-f(A'(\mathbf x'))\|_2
&\le\|f(A(\mathbf x))-f(A_*(\mathbf x))\|_2
+\|f(A_*(\mathbf x))-f(A_*'(\mathbf x'))\|_2\nonumber\\
&\quad+\|f(A_*'(\mathbf x'))-f(A'(\mathbf x'))\|_2\nonumber\\
&\le\varepsilon_{\mathrm{al}}+L_f\delta
+\varepsilon_{\mathrm{al}}=h_{\mathrm{aug}}.
\label{eq:core-semantic-bridge}
\end{align}
Since the identity belongs to $\mathcal T$, this includes the unaugmented
$\mathbf x$. For fixed
$\mathbf x\in C_k^\circ\cap\mathcal V$, average over an
independent $\mathbf X'\sim P_k$ and
$A'\sim\operatorname{Unif}(\mathcal T)$. Every representation has unit norm,
so distances involving an observation outside
$C_k^\circ\cap\mathcal V$ are at most two. Hence, by
\eqref{eq:population-classifier} and Jensen's inequality,
\begin{align}
\|f(\mathbf x)-\mathbf m_k\|_2
&=\left\|f(\mathbf x)-
\mathbb E_{\mathbf X',A'}[f(A'(\mathbf X'))]\right\|_2\nonumber\\
&=\left\|\mathbb E_{\mathbf X',A'}
[f(\mathbf x)-f(A'(\mathbf X'))]\right\|_2\nonumber\\
&\le\mathbb E_{\mathbf X',A'}
\|f(\mathbf x)-f(A'(\mathbf X'))\|_2\nonumber\\
&\le P_k(C_k^\circ\cap\mathcal V)h_{\mathrm{aug}}
+2P_k\!\left(\left\{C_k^\circ\cap
\mathcal V\right\}^c\right)\nonumber\\
&\le h_{\mathrm{aug}}+2b_{\mathrm{aug}}.
\end{align}
This proves \eqref{eq:stable-core-radius}.
For \eqref{eq:augmented-class-spread}, first keep the conditioning explicit.
Because $P_k=\operatorname{Law}(\mathbf X\mid Y=k)$ and
$\mathbf Z_0=f(A(\mathbf X))$, take independent
$\mathbf X,\mathbf X'\sim P_k$ and independent
$A,A'\sim\operatorname{Unif}(\mathcal T)$. Equation
\eqref{eq:population-classifier} equivalently gives
$\mathbf m_k=\mathbb E_{\mathbf X',A'}[f(A'(\mathbf X'))]$. Therefore,
by linearity of expectation and Jensen's inequality,
\begin{align*}
\mathbb E[\|\mathbf Z_0-\mathbf m_k\|_2\mid Y=k]
&=\mathbb E[\|f(A(\mathbf X))-\mathbf m_k\|_2\mid Y=k]\\
&=\mathbb E_{\mathbf X\sim P_k,A}\left[
\|f(A(\mathbf X))-\mathbf m_k\|_2\right]\\
&=\mathbb E_{\mathbf X\sim P_k,A}\left[
\left\|f(A(\mathbf X))-
\mathbb E_{\mathbf X',A'}[f(A'(\mathbf X'))]\right\|_2\right]\\
&=\mathbb E_{\mathbf X\sim P_k,A}\left[
\left\|\mathbb E_{\mathbf X',A'}
[f(A(\mathbf X))-f(A'(\mathbf X'))]\right\|_2\right]\\
&\le\mathbb E_{\mathbf X,\mathbf X',A,A'}
\|f(A(\mathbf X))-f(A'(\mathbf X'))\|_2.
\end{align*}
If both underlying observations lie in
$C_k^\circ\cap\mathcal V$, then
\eqref{eq:core-semantic-bridge} bounds the last distance by
$h_{\mathrm{aug}}$. Otherwise at least one observation lies outside this set,
and the unit-norm constraint bounds the distance by two. Splitting the
expectation into these two cases, and then applying the union bound, gives
\begin{align}
&\mathbb E_{\mathbf X,\mathbf X',A,A'}
\|f(A(\mathbf X))-f(A'(\mathbf X'))\|_2\nonumber\\
&\le h_{\mathrm{aug}}P\!\left(
\mathbf X\in C_k^\circ\cap\mathcal V,\ 
\mathbf X'\in C_k^\circ\cap\mathcal V
\right)\nonumber\\
&\quad+2P\!\left(
\begin{array}{c}
\mathbf X\notin C_k^\circ\cap\mathcal V\ 
\text{or}\ 
\mathbf X'\notin C_k^\circ\cap\mathcal V
\end{array}
\right)\nonumber\\
&\le h_{\mathrm{aug}}
+2P_k\!\left(\left\{C_k^\circ\cap
\mathcal V\right\}^c\right)
+2P_k\!\left(\left\{C_k^\circ\cap
\mathcal V\right\}^c\right)\nonumber\\
&\le h_{\mathrm{aug}}+4b_{\mathrm{aug}}.
\end{align}
\end{proof}

\subsubsection{Velocity Regression and Flow Stability}
The endpoint estimate below uses the following standard nonhomogeneous
integral form of Gronwall's inequality, which we state without proof.
\begin{lemma}[Integral Gronwall inequality]
\label{lem:integral-gronwall}
Let $h:[0,1]\to[0,\infty)$ be continuous, and let
$a,b:[0,1]\to[0,\infty)$ be integrable and $c\ge0$. If
\[
h(t)\le
 c+\int_0^t a(u)h(u)\,\mathrm du+
\int_0^t b(u)\,\mathrm du,
\qquad t\in[0,1],
\]
then
\[
h(t)\le
\exp\!\left(\int_0^t a(u)\,\mathrm du\right)
\left(c+\int_0^t b(s)\,\mathrm ds\right).
\]
\end{lemma}

The same endpoint estimate also exchanges an expectation and a time integral
through the following standard result, stated here without proof.
\begin{lemma}[Tonelli's theorem for nonnegative integrands]
\label{lem:tonelli}
Let $H(\omega,t)$ be a nonnegative jointly measurable function on a
probability space times $[0,1]$. Then
\[
\mathbb E\!\left[\int_0^1 H(\omega,t)\,\mathrm dt\right]
=\int_0^1\mathbb E[H(\omega,t)]\,\mathrm dt.
\]
\end{lemma}

\begin{lemma}[Velocity regression controls terminal Wasserstein error]
\label{lem:fm-endpoint}
Under Assumption~\ref{ass:regular-flow},
\begin{equation}
W_1(\widehat p_1,\nu_\varepsilon)
\le
\left(\mathbb E_{\Pi_f}
[\|\widehat{\mathbf Z}_1-\mathbf R\|_2^2]\right)^{1/2}
\le e^{L_v}
\sqrt{\frac{d^\star}{q_{\min}}\mathcal R_{\mathrm{FM}}},
\label{eq:fm-endpoint-bound}
\end{equation}
and
\begin{equation}
\|\Phi_1(\mathbf z)-\Phi_1(\mathbf z')\|_2
\le e^{L_v}\|\mathbf z-\mathbf z'\|_2.
\label{eq:flow-initial-stability}
\end{equation}
\end{lemma}
\begin{proof}
The learned trajectory and analytic path start from the same representation
$\mathbf Z_0$. By \eqref{eq:learned-ode} and
\eqref{eq:clocked-target}, respectively, their derivatives are
$\frac{\mathrm d}{\mathrm dt}\widehat{\mathbf Z}_t
=\bm v_\phi(\widehat{\mathbf Z}_t,t)$ and
$\frac{\mathrm d}{\mathrm dt}\mathbf Z_t=\mathbf U_t$.
Because both paths are absolutely continuous, the fundamental theorem of
calculus on $[0,t]$ gives their integral forms
\begin{align*}
\widehat{\mathbf Z}_t
&=\mathbf Z_0+\int_0^t
\bm v_\phi(\widehat{\mathbf Z}_u,u)\,\mathrm du,\\
\mathbf Z_t
&=\mathbf Z_0+\int_0^t\mathbf U_u\,\mathrm du.
\end{align*}
Subtracting these identities and adding and subtracting
$\bm v_\phi(\mathbf Z_u,u)$ inside the integrand gives
\begin{align*}
\widehat{\mathbf Z}_t-\mathbf Z_t
&=\int_0^t[\bm v_\phi(\widehat{\mathbf Z}_u,u)-\mathbf U_u]\,\mathrm du\\
&=\int_0^t[\bm v_\phi(\widehat{\mathbf Z}_u,u)
-\bm v_\phi(\mathbf Z_u,u)]\,\mathrm du
+\int_0^t[\bm v_\phi(\mathbf Z_u,u)-\mathbf U_u]\,\mathrm du.
\end{align*}
Taking norms, applying the triangle inequality for integrals, and then using
the spatial Lipschitz condition in Assumption~\ref{ass:regular-flow} gives
\begin{align}
\|\widehat{\mathbf Z}_t-\mathbf Z_t\|_2
&\le\int_0^t\ell_\phi(u)
\|\widehat{\mathbf Z}_u-\mathbf Z_u\|_2\,\mathrm du +\int_0^t
\|\bm v_\phi(\mathbf Z_u,u)-\mathbf U_u\|_2\,\mathrm du.
\end{align}
Applying Lemma~\ref{lem:integral-gronwall} with
$h(t)=\|\widehat{\mathbf Z}_t-\mathbf Z_t\|_2$,
$a(t)=\ell_\phi(t)$, and
$b(t)=\|\bm v_\phi(\mathbf Z_t,t)-\mathbf U_t\|_2$, $c=0$, and using
$\int_0^1\ell_\phi(u)\,\mathrm du\le L_v$ from
Assumption~\ref{ass:regular-flow}, yields
\[
\|\widehat{\mathbf Z}_t-\mathbf Z_t\|_2
\le
\exp\!\left(\int_0^t\ell_\phi(u)\,\mathrm du\right)
\int_0^t\|\bm v_\phi(\mathbf Z_u,u)-\mathbf U_u\|_2\,\mathrm du.
\]
At $t=1$, \eqref{eq:clocked-law} gives $\mathbf Z_1=\mathbf R$; hence
\begin{equation}
\|\widehat{\mathbf Z}_1-\mathbf R\|_2
\le e^{L_v}\int_0^1
\|\bm v_\phi(\mathbf Z_t,t)-\mathbf U_t\|_2\,\mathrm dt.
\label{eq:pathwise-endpoint-error}
\end{equation}
Squaring \eqref{eq:pathwise-endpoint-error}, applying Cauchy--Schwarz on
$[0,1]$, and then applying Lemma~\ref{lem:tonelli} gives
\begin{align*}
&\mathbb E_{\Pi_f}
  [\|\widehat{\mathbf Z}_1-\mathbf R\|_2^2]\\
&\le e^{2L_v}\mathbb E_{\Pi_f}\!\left[
  \left(\int_0^1
  \|\bm v_\phi(\mathbf Z_t,t)-\mathbf U_t\|_2\,\mathrm dt\right)^2
  \right]\\
&\le e^{2L_v}\mathbb E_{\Pi_f}
  \!\left[\left(\int_0^1 1^2\,\mathrm dt\right)
  \left(\int_0^1
  \|\bm v_\phi(\mathbf Z_t,t)-\mathbf U_t\|_2^2\,\mathrm dt\right)
  \right]\\
&=e^{2L_v}\int_0^1\mathbb E_{\Pi_f}
  [\|\bm v_\phi(\mathbf Z_t,t)-\mathbf U_t\|_2^2]\,\mathrm dt.
\end{align*}
On the other hand, the definition \eqref{eq:population-fm-risk} and the
density $q_T$ imply
\begin{align*}
d^\star\mathcal R_{\mathrm{FM}}
&=\mathbb E_{\Pi_f,T}
  [\|\bm v_\phi(\mathbf Z_T,T)-\mathbf U_T\|_2^2]\\
&=\int_0^1 q_T(t)\mathbb E_{\Pi_f}
  [\|\bm v_\phi(\mathbf Z_t,t)-\mathbf U_t\|_2^2]\,\mathrm dt\\
&\ge q_{\min}\int_0^1\mathbb E_{\Pi_f}
  [\|\bm v_\phi(\mathbf Z_t,t)-\mathbf U_t\|_2^2]\,\mathrm dt,
\end{align*}
where the last inequality uses the time-coverage condition
$q_T(t)\ge q_{\min}$ in Assumption~\ref{ass:regular-flow}. Combining the two
displays yields
\begin{equation}
\mathbb E_{\Pi_f}
[\|\widehat{\mathbf Z}_1-\mathbf R\|_2^2]
\le\frac{d^\star e^{2L_v}}{q_{\min}}\mathcal R_{\mathrm{FM}}.
\end{equation}
By \eqref{eq:learned-ode} and the endpoint marginal of $\Pi_f$, the pair
$(\widehat{\mathbf Z}_1,\mathbf R)$ is a coupling of $\widehat p_1$ and
$\nu_\varepsilon$.  Using this admissible coupling in the infimum defining
$W_1$ in \eqref{eq:theory-wasserstein}, followed by Cauchy--Schwarz and the
preceding endpoint second-moment bound, gives
\begin{align*}
W_1(\widehat p_1,\nu_\varepsilon)
&\le\mathbb E_{\Pi_f}
  [\|\widehat{\mathbf Z}_1-\mathbf R\|_2]\\
&\le\left(\mathbb E_{\Pi_f}
  [\|\widehat{\mathbf Z}_1-\mathbf R\|_2^2]\right)^{1/2}\\
&\le e^{L_v}
  \sqrt{\frac{d^\star}{q_{\min}}\mathcal R_{\mathrm{FM}}}.
\end{align*}
It remains to prove \eqref{eq:flow-initial-stability}.  Integrating the two
instances of the learned ODE \eqref{eq:learned-ode}, subtracting them, and
using the spatial Lipschitz condition in
Assumption~\ref{ass:regular-flow} yields
\begin{align*}
\|\Phi_t(\mathbf z)-\Phi_t(\mathbf z')\|_2
&\le\|\mathbf z-\mathbf z'\|_2+\int_0^t\ell_\phi(u)
\|\Phi_u(\mathbf z)-\Phi_u(\mathbf z')\|_2\,\mathrm du.
\end{align*}
Applying Lemma~\ref{lem:integral-gronwall} with
$c=\|\mathbf z-\mathbf z'\|_2$ and $b=0$ gives
\[
\|\Phi_t(\mathbf z)-\Phi_t(\mathbf z')\|_2
\le
\exp\!\left(\int_0^t\ell_\phi(u)\,\mathrm du\right)
\|\mathbf z-\mathbf z'\|_2
\le e^{L_v}\|\mathbf z-\mathbf z'\|_2.
\]
Taking $t=1$ proves \eqref{eq:flow-initial-stability}.
\end{proof}

\subsubsection{Terminal Class-Centre Separation}
Lemma~\ref{lem:fm-endpoint} controls the mismatch between the overall terminal
law and the reference law, but downstream classification requires a
class-wise conclusion. Under transport identifiability, the next lemma makes
this conversion: each terminal class mean is close to its assigned reference
component mean, so their pairwise geometry inherits that of the reference
components up to an error controlled by the terminal Wasserstein distance.
Here, $\Gamma^\star$ is the optimal terminal--reference coupling introduced
after \eqref{eq:theory-wasserstein}, and $\tau^\star$ is the optimal
class--component permutation defined in
\eqref{eq:optimal-class-component-permutation}.

\begin{lemma}[Leakage controls terminal class-centre geometry]
\label{lem:leakage-means}
Define the terminal and reference component means by
\begin{equation}
\widehat{\mathbf m}_k:=
\mathbb E[\widehat{\mathbf Z}_1\mid Y=k],
\qquad
\mathbf b_j:=\int\mathbf r\,Q_j^\varepsilon(\mathrm d\mathbf r),
\qquad
D:=\frac{1+2C_{\mathrm{id}}}{p_{\min}}
W_1(\widehat p_1,\nu_\varepsilon).
\label{eq:terminal-reference-means}
\end{equation}
Under Assumption~\ref{ass:transport-identifiability},
\begin{equation}
\max_k
\|\widehat{\mathbf m}_k-\mathbf b_{\tau^\star(k)}\|_2\le D,
\label{eq:leakage-mean-error}
\end{equation}
and
\begin{equation}
\max_{k\ne\ell}
\left|
\|\widehat{\mathbf m}_k-\widehat{\mathbf m}_\ell\|_2
-\|\mathbf b_{\tau^\star(k)}-\mathbf b_{\tau^\star(\ell)}\|_2
\right|\le2D.
\label{eq:terminal-distance-error}
\end{equation}
\end{lemma}

\begin{remark}[Reference center versus component mean]
\label{rem:component-mean}
Since $Q_j^\varepsilon$ is rotationally symmetric around $\mathbf c_j$,
its mean satisfies $\mathbf b_j=\alpha_\varepsilon\mathbf c_j$ for some
$\alpha_\varepsilon\in[0,1]$ that is common to all $j$. Thus $\mathbf b_j$
has the same direction as $\mathbf c_j$ but generally lies inside the unit
sphere. For example, on the unit circle with $\mathbf c_j=(1,0)$, the
symmetric points $(\cos\theta,\sin\theta)$ and
$(\cos\theta,-\sin\theta)$ average to
$(\cos\theta,0)=\cos\theta\,\mathbf c_j$. In particular,
$\alpha_\varepsilon=1$ when $\varepsilon=0$.
\end{remark}

\begin{proof}[Proof of Lemma~\ref{lem:leakage-means}]
We adapt the component-mean decomposition in
\citet{jiao2026distribution}. Recall that $p_k=P(Y=k)$, $w_j$ is the
mixture weight of $Q_j^\varepsilon$ in $\nu_\varepsilon$, and
$\pi_k(\widehat{\mathbf z})=P(Y=k\mid
\widehat{\mathbf Z}_1=\widehat{\mathbf z})$. Set
$F_k:=\mathcal S_{\tau^\star(k)}^\varepsilon$, the support of the reference
component assigned to class $k$. In the integrals below,
$\mathrm d\Gamma^\star$ abbreviates
$\Gamma^\star(\mathrm d\widehat{\mathbf z},\mathrm d\mathbf r)$.

First, the $\widehat{\mathbf Z}_1$-marginal of $\Gamma^\star$ is
$\widehat p_1$. Bayes' rule in Radon--Nikodym form gives
\begin{equation*}
\frac{\mathrm dP(\widehat{\mathbf Z}_1\in\,\cdot\mid Y=k)}
{\mathrm d\widehat p_1}(\widehat{\mathbf z})
=\frac{\pi_k(\widehat{\mathbf z})}{p_k}.
\end{equation*}
Starting from the definition of $\widehat{\mathbf m}_k$ and using this
derivative, followed by the $\widehat{\mathbf Z}_1$-marginal identity, yields
\begin{align*}
\widehat{\mathbf m}_k
&=\int\widehat{\mathbf z}\,
P(\widehat{\mathbf Z}_1\in\mathrm d\widehat{\mathbf z}\mid Y=k)\\
&=\frac1{p_k}\int\pi_k(\widehat{\mathbf z})\widehat{\mathbf z}\,
\widehat p_1(\mathrm d\widehat{\mathbf z})\\
&=\frac1{p_k}\int\pi_k(\widehat{\mathbf z})\widehat{\mathbf z}\,
\mathrm d\Gamma^\star.
\end{align*}
Equivalently,
$p_k\widehat{\mathbf m}_k
=\int\pi_k(\widehat{\mathbf z})\widehat{\mathbf z}\,
\mathrm d\Gamma^\star$.
Second, the $\mathbf R^\star$-marginal of $\Gamma^\star$ is
$\nu_\varepsilon=\sum_jw_jQ_j^\varepsilon$. Since the component supports are
disjoint, the relevant Radon--Nikodym derivative is
\begin{equation*}
\frac{\mathrm dQ_{\tau^\star(k)}^\varepsilon}
{\mathrm d\nu_\varepsilon}(\mathbf r)
=\frac{\mathbf 1\{\mathbf r\in F_k\}}
{w_{\tau^\star(k)}}
\qquad \nu_\varepsilon\text{-a.e.}
\end{equation*}
Starting from the definition of $\mathbf b_{\tau^\star(k)}$, using this
derivative and then the $\mathbf R^\star$-marginal identity gives
\begin{align*}
\mathbf b_{\tau^\star(k)}
&=\int\mathbf r\,
Q_{\tau^\star(k)}^\varepsilon(\mathrm d\mathbf r)\\
&=\frac1{w_{\tau^\star(k)}}
\int_{F_k}\mathbf r\,\nu_\varepsilon(\mathrm d\mathbf r)\\
&=\frac1{w_{\tau^\star(k)}}
\int\mathbf 1\{\mathbf r\in F_k\}\mathbf r\,
\mathrm d\Gamma^\star.
\end{align*}
Equivalently,
$w_{\tau^\star(k)}\mathbf b_{\tau^\star(k)}
=\int\mathbf 1\{\mathbf r\in F_k\}\mathbf r\,
\mathrm d\Gamma^\star$.
Subtracting these two identities, and then adding and subtracting
$\pi_k(\widehat{\mathbf z})\mathbf r$ inside the integrand, yields
\begin{align}
p_k\widehat{\mathbf m}_k
-w_{\tau^\star(k)}\mathbf b_{\tau^\star(k)}
&=\int\left[
\pi_k(\widehat{\mathbf z})\widehat{\mathbf z}
-\mathbf 1\{\mathbf r\in F_k\}\mathbf r
\right]\,\mathrm d\Gamma^\star\nonumber\\
&=\int\pi_k(\widehat{\mathbf z})
(\widehat{\mathbf z}-\mathbf r)\,\mathrm d\Gamma^\star
\nonumber\\
&\quad+\int\left[\pi_k(\widehat{\mathbf z})
-\mathbf 1\{\mathbf r\in F_k\}\right]
\mathbf r\,\mathrm d\Gamma^\star.
\label{eq:leakage-moment-decomposition}
\end{align}
Write $\ell_k:=p_k+w_{\tau^\star(k)}-2q_k^\star$. Since
$\pi_k(\widehat{\mathbf z})\in[0,1]$ and
$\mathbf 1\{\mathbf r\in F_k\}\in\{0,1\}$, pointwise we have
\begin{align*}
&\left|\pi_k(\widehat{\mathbf z})
-\mathbf 1\{\mathbf r\in F_k\}\right|=\pi_k(\widehat{\mathbf z})
+\mathbf 1\{\mathbf r\in F_k\}
-2\pi_k(\widehat{\mathbf z})\mathbf 1\{\mathbf r\in F_k\}.
\end{align*}
Recall from the two marginals of $\Gamma^\star$ and
\eqref{eq:label-component-mass} that
\begin{gather*}
\int\pi_k(\widehat{\mathbf z})\,\mathrm d\Gamma^\star
=p_k,\\
\int\mathbf 1\{\mathbf r\in F_k\}\,\mathrm d\Gamma^\star
=w_{\tau^\star(k)},\\
\int\pi_k(\widehat{\mathbf z})
\mathbf 1\{\mathbf r\in F_k\}\,\mathrm d\Gamma^\star
=q_k^\star.
\end{gather*}
Therefore, integrating the pointwise identity gives
\begin{align*}
&\int\left|\pi_k(\widehat{\mathbf z})
-\mathbf 1\{\mathbf r\in F_k\}\right|\,\mathrm d\Gamma^\star=p_k+w_{\tau^\star(k)}-2q_k^\star=\ell_k.
\end{align*}
Consequently, the second term in
\eqref{eq:leakage-moment-decomposition} satisfies
\begin{align*}
&\left\|\int\left[\pi_k(\widehat{\mathbf z})
-\mathbf 1\{\mathbf r\in F_k\}\right]
\mathbf r\,\mathrm d\Gamma^\star\right\|_2\le\int\left|\pi_k(\widehat{\mathbf z})
-\mathbf 1\{\mathbf r\in F_k\}\right|
\|\mathbf r\|_2\,\mathrm d\Gamma^\star
=\ell_k,
\end{align*}
where the last equality uses $\|\mathbf r\|_2=1$. Similarly, the first
term satisfies
\begin{align*}
\left\|\int\pi_k(\widehat{\mathbf z})
(\widehat{\mathbf z}-\mathbf r)\,\mathrm d\Gamma^\star\right\|_2&\le\int\pi_k(\widehat{\mathbf z})
\|\widehat{\mathbf z}-\mathbf r\|_2\,\mathrm d\Gamma^\star\\
&\le\int\|\widehat{\mathbf z}-\mathbf r\|_2\,
\mathrm d\Gamma^\star
=W_1(\widehat p_1,\nu_\varepsilon),
\end{align*}
because $\Gamma^\star$ is an optimal coupling. Combining the two bounds yields
\begin{equation}
\|p_k\widehat{\mathbf m}_k
-w_{\tau^\star(k)}\mathbf b_{\tau^\star(k)}\|_2
\le W_1(\widehat p_1,\nu_\varepsilon)+\ell_k.
\end{equation}
Moreover, $|p_k-w_{\tau^\star(k)}|\le\ell_k$ and
$\|\mathbf b_{\tau^\star(k)}\|_2\le1$. Consequently,
\begin{align}
p_k\|\widehat{\mathbf m}_k-\mathbf b_{\tau^\star(k)}\|_2
&\le W_1(\widehat p_1,\nu_\varepsilon)+2\ell_k\nonumber\\
&\le(1+2C_{\mathrm{id}})W_1(\widehat p_1,\nu_\varepsilon).
\end{align}
Division by $p_k\ge p_{\min}$ proves \eqref{eq:leakage-mean-error}.
Finally, for $k\ne\ell$, the reverse triangle inequality followed by
\eqref{eq:leakage-mean-error} gives
\begin{align*}
\left|
\|\widehat{\mathbf m}_k-\widehat{\mathbf m}_\ell\|_2
-\|\mathbf b_{\tau^\star(k)}-\mathbf b_{\tau^\star(\ell)}\|_2
\right|
&\le\| (\widehat{\mathbf m}_k-\mathbf b_{\tau^\star(k)})
-(\widehat{\mathbf m}_\ell-\mathbf b_{\tau^\star(\ell)})\|_2\\
&\le
\|\widehat{\mathbf m}_k-\mathbf b_{\tau^\star(k)}\|_2
+\|\widehat{\mathbf m}_\ell-\mathbf b_{\tau^\star(\ell)}\|_2\\
&\le2D.
\end{align*}
This proves \eqref{eq:terminal-distance-error} and completes the proof.
\end{proof}

Since $\|\mathbf b_j-\mathbf c_j\|_2\le r_\varepsilon$,
Lemma~\ref{lem:leakage-means} gives the terminal separation bound
\begin{equation}
\min_{k\ne\ell}
\|\widehat{\mathbf m}_k-\widehat{\mathbf m}_\ell\|_2
\ge\Delta_{\mathcal C}-2r_\varepsilon-2D.
\label{eq:terminal-class-separation}
\end{equation}
By Lemma~\ref{lem:fm-endpoint},
\begin{equation}
D\le
\frac{(1+2C_{\mathrm{id}})e^{L_v}}{p_{\min}}
\sqrt{\frac{d^\star}{q_{\min}}\mathcal R_{\mathrm{FM}}}.
\label{eq:fm-leakage-mean-bound}
\end{equation}

\subsubsection{Downstream Classification Guarantee}
We now restate Theorem~\ref{thm:population-error} with the explicit constants
used in its proof.
\begin{mytheorem}[Population FBDM loss controls classification error]
Suppose Assumptions~\ref{ass:augmentation-quality}--\ref{ass:transport-identifiability}
hold. Fix a view-stability tolerance $\varepsilon_{\mathrm{al}}>0$ such that
\begin{equation*}
e^{-L_v}(\Delta_{\mathcal C}-2r_\varepsilon)
>4L_f\delta+8\varepsilon_{\mathrm{al}}+12(1-\sigma).
\end{equation*}
Define
\begin{align}
\Delta_0
&:=e^{-L_v}(\Delta_{\mathcal C}-2r_\varepsilon),\qquad
C_{\mathrm{err}}:=\frac{N_{\mathrm{aug}}}{\varepsilon_{\mathrm{al}}}
\sqrt{\frac{d^\star}{\lambda}},
\label{eq:classification-constants}\\
C_{\mathrm{sep}}
&:=\frac{2(1+2C_{\mathrm{id}})}{p_{\min}}
\sqrt{\frac{d^\star}{q_{\min}}}
+\frac{12C_{\mathrm{err}}}{p_{\min}},
\label{eq:separation-constant}\\
g_{\mathrm{aug}}
&:=\Delta_0-4L_f\delta-8\varepsilon_{\mathrm{al}}-12(1-\sigma).
\label{eq:augmentation-margin}
\end{align}

Then $g_{\mathrm{aug}}>0$. Set
\begin{equation}
r_0:=\left(\frac{g_{\mathrm{aug}}}{2C_{\mathrm{sep}}}\right)^2.
\label{eq:small-risk-constants}
\end{equation}
These constants depend only on the stated uniform bounds.
If $\mathcal R_{\mathrm{FBDM}}(f,\phi)\le r_0$, then
\begin{equation*}
\operatorname{Err}(G_f)\le(1-\sigma)+
C_{\mathrm{err}}\sqrt{\mathcal R_{\mathrm{FBDM}}(f,\phi)}.
\end{equation*}
\end{mytheorem}

\begin{proof}[Proof of Theorem~\ref{thm:population-error}]
By construction and \eqref{eq:population-classifier},
$\operatorname{Law}(\mathbf Z_0\mid Y=k)=f_\#P_{\mathcal A,k}$ and
$\mathbf m_k=\mathbb E[\mathbf Z_0\mid Y=k]$, where $\mathbf Z_0$ is the
source representation in the coupling $\Pi_f$. Let $\Phi_1$ denote the
terminal map of the learned flow and define the corresponding terminal
class mean by
$\widehat{\mathbf m}_k:=\mathbb E[\Phi_1(\mathbf Z_0)\mid Y=k]$.
For readability, recall the three quantities that measure within-class
spread:
\begin{gather}
\mathcal V
:=\left\{\mathbf x\in\operatorname{supp}(P_X):
\max_{A,A'\in\mathcal T}
\|f(A(\mathbf x))-f(A'(\mathbf x))\|_2
\le\varepsilon_{\mathrm{al}}\right\},
\nonumber\\
\mathcal U:=P_X(\mathcal V^c),
\nonumber\\
h_{\mathrm{aug}}:=L_f\delta+2\varepsilon_{\mathrm{al}},
\nonumber\\
b_{\mathrm{aug}}:=1-\sigma+\frac{\mathcal U}{p_{\min}}.
\label{eq:proof-spread-constants}
\end{gather}
Here $\mathcal U$ is the mass of observations whose augmented views are
not $\varepsilon_{\mathrm{al}}$-stable, $h_{\mathrm{aug}}$ bounds the
diameter of a stable class core, and $b_{\mathrm{aug}}$ bounds the mass not
covered by that core. These are exactly the quantities entering the
class-spread bounds in Lemma~\ref{lem:augmentation-core-distances}.
The next bound compares evolving the class mean, $\Phi_1(\mathbf m_k)$,
with averaging the individually evolved representations, $\widehat{\mathbf m}_k$.
Since $\Phi_1(\mathbf m_k)$ is deterministic conditional on $Y=k$,
conditional Jensen's inequality gives
\begin{align}
\|\Phi_1(\mathbf m_k)-\widehat{\mathbf m}_k\|_2
&=
\left\|
\mathbb E\!\left[
\Phi_1(\mathbf m_k)-\Phi_1(\mathbf Z_0)\mid Y=k
\right]
\right\|_2\nonumber\\
&\le
\mathbb E[\|\Phi_1(\mathbf m_k)-\Phi_1(\mathbf Z_0)\|_2\mid Y=k]
\nonumber\\
&\le e^{L_v}
\mathbb E[\|\mathbf m_k-\mathbf Z_0\|_2\mid Y=k]
\nonumber\\
&\le e^{L_v}(h_{\mathrm{aug}}+4b_{\mathrm{aug}}).
\label{eq:transported-class-mean}
\end{align}
Here the penultimate inequality uses the flow stability in
Lemma~\ref{lem:fm-endpoint}, and the last uses the class-spread bound from
Lemma~\ref{lem:augmentation-core-distances}.
Although $\mathbf m_k$ may lie inside the unit sphere,
$\Phi_1(\mathbf m_k)$ is well-defined because the flow acts on
$\mathbb R^{d^\star}$; \eqref{eq:transported-class-mean} controls its
discrepancy from $\widehat{\mathbf m}_k$.

On the reference side, define
\begin{equation}
\mathbf b_j:=\int\mathbf r\,Q_j^\varepsilon(\mathrm d\mathbf r),
\qquad
D:=\frac{1+2C_{\mathrm{id}}}{p_{\min}}
W_1(\widehat p_1,\nu_\varepsilon).
\label{eq:proof-terminal-reference-constants}
\end{equation}
Here $\mathbf b_j$ is the mean of reference component $j$, and Lemma~
\ref{lem:leakage-means} shows that $D$ uniformly bounds
$\|\widehat{\mathbf m}_k-\mathbf b_{\tau^\star(k)}\|_2$, where
$\tau^\star(k)$ is the reference component matched to class $k$. Also recall
the reference-centre gap and component radius
\begin{equation}
\Delta_{\mathcal C}:=\min_{i\ne j}\|\mathbf c_i-\mathbf c_j\|_2,
\qquad
r_\varepsilon:=\max_j\sup_{\mathbf r\in\mathcal S_j^\varepsilon}
\|\mathbf r-\mathbf c_j\|_2.
\label{eq:proof-reference-geometry}
\end{equation}
The radius definition implies
$\|\mathbf b_j-\mathbf c_j\|_2\le r_\varepsilon$.
For $k\ne\ell$, Lemma~\ref{lem:leakage-means} thus gives
\begin{align}
\Delta_{\mathcal C}-2r_\varepsilon
&\le
\|\mathbf b_{\tau^\star(k)}-\mathbf b_{\tau^\star(\ell)}\|_2\nonumber\\
&\le2D+
\|\widehat{\mathbf m}_k-\widehat{\mathbf m}_\ell\|_2\nonumber\\
&\le2D+e^{L_v}
\{2h_{\mathrm{aug}}+8b_{\mathrm{aug}}
+\|\mathbf m_k-\mathbf m_\ell\|_2\}.
\end{align}
Define the initial class-mean gap and the reference gap that remains after
component perturbation and inverse-flow contraction by
\begin{equation}
\Delta_f:=\min_{k\ne\ell}\|\mathbf m_k-\mathbf m_\ell\|_2,
\qquad
\Delta_0:=e^{-L_v}(\Delta_{\mathcal C}-2r_\varepsilon).
\label{eq:proof-separation-constants}
\end{equation}
The preceding display then yields
\begin{equation}
\Delta_f
\ge\Delta_0-2e^{-L_v}D-2h_{\mathrm{aug}}-8b_{\mathrm{aug}}.
\label{eq:class-mean-separation}
\end{equation}
To compare this separation with the stable-core radius, define
\begin{align}
g_{\mathrm{aug}}
&:=\Delta_0-4L_f\delta-8\varepsilon_{\mathrm{al}}-12(1-\sigma),
\nonumber\\
C_{\mathrm{sep}}
&:=\frac{2(1+2C_{\mathrm{id}})}{p_{\min}}
\sqrt{\frac{d^\star}{q_{\min}}}
+\frac{12C_{\mathrm{err}}}{p_{\min}}.
\label{eq:proof-margin-constants}
\end{align}
Thus $g_{\mathrm{aug}}$ is the fixed separation margin left after the
augmentation terms, whereas $C_{\mathrm{sep}}$ collects the coefficients
of the loss-dependent errors. Here $p_{\min}=\min_kP(Y=k)$,
$q_{\min}$ is the lower bound on the flow-time density, and
$C_{\mathrm{id}}$ is the identifiability constant in
Assumption~\ref{ass:transport-identifiability}; $C_{\mathrm{err}}$ is given
in Theorem~\ref{thm:population-error} above.
To prove that an image is closer to its own class mean, fix
$\mathbf x\in C_k^\circ\cap\mathcal V$ and $\ell\ne k$.
We show that the difference between its distances to the other class mean
and its own class mean is positive. The triangle inequality,
the within-class radius bound \eqref{eq:stable-core-radius}, and the
class-mean separation bound \eqref{eq:class-mean-separation} give the
first three inequalities below. We then substitute the bounds
\eqref{eq:alignment-stability-bound} and
\eqref{eq:fm-leakage-mean-bound}:
\begin{align}
\|f(\mathbf x)-\mathbf m_\ell\|_2
-\|f(\mathbf x)-\mathbf m_k\|_2&\ge\Delta_f-2\|f(\mathbf x)-\mathbf m_k\|_2\nonumber\\
&\ge\Delta_f-2(h_{\mathrm{aug}}+2b_{\mathrm{aug}})\nonumber\\
&\ge\Delta_0-2e^{-L_v}D-4h_{\mathrm{aug}}-12b_{\mathrm{aug}}
\nonumber\\
&=\Delta_0-2e^{-L_v}D
-4(L_f\delta+2\varepsilon_{\mathrm{al}})
-12\left(1-\sigma+\frac{\mathcal U}{p_{\min}}\right)
\nonumber\\
&=\left\{\Delta_0-4L_f\delta-8\varepsilon_{\mathrm{al}}
-12(1-\sigma)\right\}
-2e^{-L_v}D-\frac{12\mathcal U}{p_{\min}}
\nonumber\\
&=g_{\mathrm{aug}}-2e^{-L_v}D
-\frac{12\mathcal U}{p_{\min}}
\nonumber\\
&\ge g_{\mathrm{aug}}
-\frac{2(1+2C_{\mathrm{id}})}{p_{\min}}
\sqrt{\frac{d^\star}{q_{\min}}\mathcal R_{\mathrm{FM}}}
-\frac{12\mathcal U}{p_{\min}}\nonumber\\
&\ge g_{\mathrm{aug}}
-\frac{2(1+2C_{\mathrm{id}})}{p_{\min}}
\sqrt{\frac{d^\star}{q_{\min}}\mathcal R_{\mathrm{FM}}}
-\frac{12C_{\mathrm{err}}}{p_{\min}}
\sqrt{\mathcal R_{\mathrm{FBDM}}}
\nonumber\\
&\ge g_{\mathrm{aug}}-C_{\mathrm{sep}}
\sqrt{\mathcal R_{\mathrm{FBDM}}}
\ge\frac{g_{\mathrm{aug}}}{2}>0.
\label{eq:classification-margin-from-loss}
\end{align}
The first equality substitutes the definitions of $h_{\mathrm{aug}}$ and
$b_{\mathrm{aug}}$ from \eqref{eq:proof-spread-constants}; the next equality
identifies the bracketed fixed margin as $g_{\mathrm{aug}}$. The following
two inequalities respectively use the bounds on $D$ in
\eqref{eq:fm-leakage-mean-bound} and on $\mathcal U$ in
\eqref{eq:alignment-stability-bound}. Finally,
$\mathcal R_{\mathrm{FM}}\le\mathcal R_{\mathrm{FBDM}}$ and the definition of
$C_{\mathrm{sep}}$ collect both loss-dependent terms, while
\eqref{eq:small-risk-constants} gives the last inequality. Therefore,
for every $\ell\ne k$,
\begin{equation}
\|f(\mathbf x)-\mathbf m_\ell\|_2
>\|f(\mathbf x)-\mathbf m_k\|_2.
\end{equation}
Hence $G_f(\mathbf x)=k$. Misclassification is confined to observations
outside their class core or outside the stable-view set:
\begin{align}
\operatorname{Err}(G_f)
&\le\sum_{k=1}^K p_k P_k((C_k^\circ)^c)+\mathcal U
\nonumber\\
&\le(1-\sigma)+\mathcal U
\le(1-\sigma)+C_{\mathrm{err}}\sqrt{\mathcal R_{\mathrm{FBDM}}}.
\label{eq:full-classification-bound}
\end{align}
This proves \eqref{eq:population-misclassification}.
The resulting error bound depends on augmentation coverage and the FBDM risk,
with the stated constants fixing the admissible loss threshold.
\end{proof}

\subsection{A Geometric Sufficient Condition for Leakage Control}
\label{app:leakage-sufficient}

Theorem~\ref{thm:population-error} uses
Assumption~\ref{ass:transport-identifiability} to turn small terminal
$W_1$ discrepancy into small class--component leakage. The purpose of this
section is to give a concrete geometric condition under which that implication
holds. The condition allows a controlled fraction of terminal mass to lie in
or near an off-matched reference component.

Recall that
$\mathcal S_j^\varepsilon:=\operatorname{supp}(Q_j^\varepsilon)$ is the
support of reference component $j$. For a distance threshold $\Delta>0$,
define
\begin{equation}
q_{kj}^{\star,<\Delta}:=
\int \pi_k(\widehat{\mathbf z})
\mathbf 1\{\mathbf r\in\mathcal S_j^\varepsilon\}
\mathbf 1\{\|\widehat{\mathbf z}-\mathbf r\|_2<\Delta\}\,
\Gamma^\star(\mathrm d\widehat{\mathbf z},\mathrm d\mathbf r).
\label{eq:near-label-component-mass}
\end{equation}
Thus $q_{kj}^{\star,<\Delta}$ is the part of $q_{kj}^\star$ whose terminal
representation is within $\Delta$ of its coupled target in component $j$.
When $j\ne\tau^\star(k)$, it measures class-$k$ mass that is paired with an
off-matched component and is already inside or close to that component.

\begin{proposition}[A soft off-component margin controls leakage]
\label{prop:leakage-gap}
Suppose that, for some $\Delta_{\mathrm{sep}}>0$ and
$\alpha_{\mathrm{sep}}\in[0,1)$,
\begin{equation}
\sum_{k=1}^K\sum_{j\ne\tau^\star(k)}
q_{kj}^{\star,<\Delta_{\mathrm{sep}}}
\le
\alpha_{\mathrm{sep}}
\sum_{k=1}^K\sum_{j\ne\tau^\star(k)}q_{kj}^\star.
\label{eq:soft-off-component-margin}
\end{equation}
Then
\begin{equation}
\operatorname{Leak}(\Gamma^\star)
\le\frac{2}{(1-\alpha_{\mathrm{sep}})\Delta_{\mathrm{sep}}}
W_1(\widehat p_1,\nu_\varepsilon).
\label{eq:soft-margin-leakage-bound}
\end{equation}
\end{proposition}

Condition~\eqref{eq:soft-off-component-margin} allows some off-matched mass
to enter or approach a wrong component, but limits its proportion to
$\alpha_{\mathrm{sep}}$. Hence at least a fraction
$1-\alpha_{\mathrm{sep}}$ of the off-matched mass pays transport cost no less
than $\Delta_{\mathrm{sep}}$, allowing $W_1$ to control the leakage.

\begin{proof}
We first bound the total off-matched mass
\begin{equation}
\sum_{k=1}^K\sum_{j\ne\tau^\star(k)}q_{kj}^\star
\label{eq:total-off-component-mass}
\end{equation}
using $W_1$. Because $\Gamma^\star$ is an optimal coupling,
\begin{align}
W_1(\widehat p_1,\nu_\varepsilon)
&=\int \|\widehat{\mathbf z}-\mathbf r\|_2\,
\Gamma^\star(\mathrm d\widehat{\mathbf z},\mathrm d\mathbf r)\nonumber\\
&\ge \sum_{k=1}^K\sum_{j\ne\tau^\star(k)}
\int \pi_k(\widehat{\mathbf z})
\mathbf 1\{\mathbf r\in\mathcal S_j^\varepsilon\}
\mathbf 1\{\|\widehat{\mathbf z}-\mathbf r\|_2
\ge\Delta_{\mathrm{sep}}\}
\|\widehat{\mathbf z}-\mathbf r\|_2\,
\Gamma^\star(\mathrm d\widehat{\mathbf z},\mathrm d\mathbf r)\nonumber\\
&\ge\Delta_{\mathrm{sep}}
\sum_{k=1}^K\sum_{j\ne\tau^\star(k)}
\left(q_{kj}^\star-q_{kj}^{\star,<\Delta_{\mathrm{sep}}}\right)
\nonumber\\
&\ge(1-\alpha_{\mathrm{sep}})\Delta_{\mathrm{sep}}
\sum_{k=1}^K\sum_{j\ne\tau^\star(k)}q_{kj}^\star.
\label{eq:off-mass-w1-bound}
\end{align}
Here the first inequality restricts a nonnegative integral to off-matched
pairs at distance at least $\Delta_{\mathrm{sep}}$; the second uses the
definition \eqref{eq:near-label-component-mass}; and the last uses
\eqref{eq:soft-off-component-margin}.

To see why \eqref{eq:total-off-component-mass} controls leakage, view
$Q^\star=(q_{kj}^\star)_{k,j}$ as a class--component mass matrix. After
matching class $k$ with column $\tau^\star(k)$, the entries
$q_{k,\tau^\star(k)}^\star$ form its matched diagonal, and
\eqref{eq:total-off-component-mass} sums all off-diagonal entries. Since the
row and column sums are $p_k=\sum_jq_{kj}^\star$ and
$w_j=\sum_iq_{ij}^\star$, class-$k$ leakage is one off-diagonal row sum plus
the corresponding off-diagonal column sum. Each is bounded by the sum of all
off-diagonal entries; hence
\begin{align}
p_k+w_{\tau^\star(k)}-2q_{k,\tau^\star(k)}^\star
&=\sum_{j\ne\tau^\star(k)}q_{kj}^\star
+\sum_{i\ne k}q_{i,\tau^\star(k)}^\star\nonumber\\
&\le2\sum_{i=1}^K\sum_{j\ne\tau^\star(i)}q_{ij}^\star.
\label{eq:row-column-leakage}
\end{align}
Taking the maximum over $k$ in \eqref{eq:row-column-leakage} and using
\eqref{eq:off-mass-w1-bound} gives
\begin{equation}
\operatorname{Leak}(\Gamma^\star)
\le \frac{2}{(1-\alpha_{\mathrm{sep}})\Delta_{\mathrm{sep}}}
W_1(\widehat p_1,\nu_\varepsilon).
\end{equation}
This proves the claim.
\end{proof}

Thus Assumption~\ref{ass:transport-identifiability} holds with
$C_{\mathrm{id}}=2/[(1-\alpha_{\mathrm{sep}})\Delta_{\mathrm{sep}}]$.

%% file: appendices/implementation.tex
\section{Implementation and Experimental Protocols}
\label{app:implementation}
\subsection{Training Refinements and Velocity-Energy Regularization}
\label{app:grad-scale}

\paragraph{Two-timescale gradient refinement.}
Unlike conventional flow matching, whose source law is fixed, FBDM jointly
learns the source representations and the velocity field. Encoder updates thus
change the states, paths, and target velocities defining the regression problem.
For the four reported ResNet-18 benchmarks, we let $\bm v_\phi$ track the
current representation geometry before that geometry moves substantially:
we use $\rho=0.75$ to attenuate only the FM gradient entering the encoder and
projector, while retaining the full velocity-field and alignment updates.
The ImageNet-1K run uses $\rho=1$ and does not apply this gradient refinement.

For a scale $0<\rho\leq1$, let $\sg$ denote stop-gradient and define the
forward-identical surrogate
\begin{equation}
\widetilde{\mathbf z}
=\sg(\mathbf z)+\rho(\mathbf z-\sg(\mathbf z)).
\label{eq:grad-scale}
\end{equation}
The surrogate replaces $\mathbf z$ in both the FM state and its analytic target;
alignment uses the original representation and assignment stays detached.
It satisfies $\widetilde{\mathbf z}=\mathbf z$ in the forward pass and
$\partial\widetilde{\mathbf z}/\partial\mathbf z=\rho\mathbf I$. Consequently,
for the ResNet-18 objective \eqref{eq:objective}, the training gradients are
\begin{align}
\nabla_{\theta,\psi}^{\mathrm{train}}\widehat{\mathcal R}_B
&=\rho\,
\nabla_{\theta,\psi}\widehat{\mathcal R}_{\mathrm{FM},B}
+\lambda\,
\nabla_{\theta,\psi}\widehat{\mathcal R}_{\mathrm{align},B},\\
\nabla_\phi^{\mathrm{train}}\widehat{\mathcal R}_B
&=\nabla_\phi\widehat{\mathcal R}_{\mathrm{FM},B}.
\label{eq:scaled-gradients}
\end{align}
Thus $\rho$ scales the representation-side FM gradient while preserving the
field and alignment updates.
Algorithm~\ref{alg:fbdm} gives the core loop. Main benchmark results use fixed
evaluation; online-only diagnostics are explicitly identified in
Appendix~\ref{app:alignment-optimization}.

\paragraph{Velocity-energy regularization.}
The velocity field is auxiliary during pretraining, whereas the encoder is
retained downstream. ImageNet's initial assignment directions have a shared
directional bias and an anisotropic covariance spectrum
(Appendix~\ref{app:initial-rank}).
To encourage the source representations to adapt toward the reference geometry,
we regularize the auxiliary velocity magnitude. Only for ImageNet-1K, we add
to \eqref{eq:fm} and \eqref{eq:alignment} a coordinate-averaged energy penalty
that updates only the velocity field:
\[
\widehat{\mathcal R}^{\mathrm{ImageNet}}_B
=1.3\,\widehat{\mathcal R}_{\mathrm{FM},B}
+10\,\widehat{\mathcal R}_{\mathrm{align},B}
+\frac{1.3\beta}{d^\star}\,\mathbb E_{\mathbf Z_T,T}
\!\left[\|\bm v_\phi(\sg(\mathbf Z_T),T)\|_2^2\right].
\]
The reported 100-epoch ImageNet-1K run uses $\beta=20$; the four ResNet-18
benchmarks omit the energy penalty. A 50-epoch ImageNet-1K comparison in
Appendix~\ref{app:velocity-energy-ablation} evaluates its empirical contribution.
Appendix~\ref{app:direct-regression} (Table~\ref{tab:direct-regression-main})
compares this velocity-mediated objective with direct regression to the same
assigned targets.

\paragraph{Velocity energy and initial distribution matching.}
The following bound concerns the starting representation law itself and allows
the assigned endpoint marginal to differ from the designed reference.

\begin{proposition}[Velocity energy and initial distribution matching]
\label{prop:velocity-energy-matching}
Fix a representation--endpoint coupling
$\Pi_f=\operatorname{Law}(\mathbf Z_0,\mathbf R)$ on
$\Sphere^{d^\star-1}\times\Sphere^{d^\star-1}$ with non-antipodal endpoints
almost surely, and write $\mu_f=\operatorname{Law}(\mathbf Z_0)$ and
$\nu_{\mathrm{assigned}}=\operatorname{Law}(\mathbf R)$.
Use the ideal shortest paths \eqref{eq:clocked-law} and velocities
\eqref{eq:clocked-target}, with $T\sim q_T$ independent of the endpoints and
$0<\kappa_s:=\mathbb E[\dot s(T)^2]<\infty$.
For a measurable field, define both coordinate-averaged quantities under this
same coupling and time law:
\[
\mathcal R_{\mathrm{FM}}
=\frac1{d^\star}\mathbb E_{\Pi_f,T}
\|\bm v_\phi(\mathbf Z_T,T)-\mathbf U_T\|_2^2,
\qquad
\mathcal K_\phi
=\frac1{d^\star}\mathbb E_{\Pi_f,T}
\|\bm v_\phi(\mathbf Z_T,T)\|_2^2.
\]
If these quantities are finite, then
\[
W_{2,\Sphere}(\mu_f,\nu_{\mathrm{assigned}})
\le\sqrt{\frac{d^\star}{\kappa_s}}
\left(\sqrt{\mathcal R_{\mathrm{FM}}}+\sqrt{\mathcal K_\phi}\right),
\]
where $W_{2,\Sphere}$ uses squared spherical geodesic distance as its cost.
For the reported ImageNet rational clock with $a=0.5$ and path-uniform
sampling, $\kappa_s=211/180$.
If $\nu_{\mathrm{assigned}}=\sum_{k=1}^{K'}\widetilde w_kQ_k^\varepsilon$
and $\nu_\varepsilon=\sum_{k=1}^{K'}w_kQ_k^\varepsilon$, set
$\delta_{\mathrm{mix}}=\frac12\sum_k|\widetilde w_k-w_k|$. Then the
Euclidean $W_1$ satisfies
\[
W_1(\mu_f,\nu_\varepsilon)
\le\sqrt{\frac{d^\star}{\kappa_s}}
\left(\sqrt{\mathcal R_{\mathrm{FM}}}+\sqrt{\mathcal K_\phi}\right)
+2\delta_{\mathrm{mix}}.
\]
\end{proposition}
\begin{proof}
Let $\omega=\arccos(\mathbf Z_0^\top\mathbf R)$. The constant path-progress
speed in Appendix~\ref{app:derive-slerp} and independence of $T$ give
$\mathbb E\|\mathbf U_T\|_2^2=\kappa_s\mathbb E[\omega^2]$.
The coupling $\Pi_f$ is admissible for $W_{2,\Sphere}$, and the $L^2$
triangle inequality yields
\[
\sqrt{\kappa_s}\,W_{2,\Sphere}(\mu_f,\nu_{\mathrm{assigned}})
\le\sqrt{\mathbb E\|\mathbf U_T\|_2^2}
\le\sqrt{d^\star}
\left(\sqrt{\mathcal R_{\mathrm{FM}}}+\sqrt{\mathcal K_\phi}\right).
\]
For path-uniform sampling, $q_T=\dot s$, so the rational clock gives
$\kappa_s=\int_0^1[(1+a)/(1+at)^2]^3\,\mathrm dt=211/180$ at $a=0.5$.
Coupling the common mass of each mixture component leaves mass
$\delta_{\mathrm{mix}}$ to move within a set of Euclidean diameter $2$;
hence $W_1(\nu_{\mathrm{assigned}},\nu_\varepsilon)\le2\delta_{\mathrm{mix}}$.
The last bound follows by the triangle inequality and
$W_1(\mu_f,\nu_{\mathrm{assigned}})\le
W_{2,\Sphere}(\mu_f,\nu_{\mathrm{assigned}})$.
\end{proof}
The implemented energy reuses the FM states and times; stop-gradient preserves
its forward value while restricting its gradient to the field.
Here $d^\star=512$ for ImageNet is the normalized projector dimension.
The proposition uses ideal paths; normalized linear interpolation near
coincident endpoints and angle clamping are numerical approximations in the
implementation. Capacity limits do not enforce equal mixture weights, so
$\delta_{\mathrm{mix}}$ need not vanish.

\subsection{Evaluation and Configuration}
\label{app:configs}

\paragraph{Evaluation protocol.}

For FBDM on the four benchmarks in Table~\ref{tab:ssl-comparison}, ResNet-18
is pretrained for 1000 epochs. The fixed linear protocol trains a classifier
for 500 epochs with probe seeds $\{0,1,2\}$ on the frozen $512$-dimensional
backbone features; cosine $5$-NN is deterministic for a fixed feature bank.
The separate ImageNet-1K comparison in
Table~\ref{tab:imagenet-comparison} uses ResNet-50 features after 100 epochs of
pretraining with global batch size 512. For FBDM, the linear probe uses all
1,281,167 ImageNet-1K training images and reports top-1 accuracy on all 50,000
validation images (top-5 is also computed). Both splits use
\texttt{Resize(256)} $\rightarrow$ \texttt{CenterCrop(224)} $\rightarrow$
\texttt{ToTensor()}, without input normalization or random augmentation.
The ResNet-50 backbone is frozen with batch normalization in evaluation mode;
its 2048-dimensional features are taken before the projection head, without
using the velocity field or ODE endpoint. A \texttt{Linear(2048,1000)} classifier
is trained with cross-entropy for 500 epochs using Adam, initial learning rate
$0.01$ exponentially decayed each epoch to $10^{-6}$, weight decay
$5\times10^{-6}$, batch size 1000, and seed 0.

\paragraph{Training hyperparameters.}

For the dynamic-capacity variant described in Section~\ref{sec:method}, we
write the training-set-specific schedule as $m_{\mathcal D_N}(e)$.

\begin{table}[ht]
\centering
\caption{Dataset-specific assignment capacities used by the reported FBDM
runs. Here $B$ is the minibatch size and $m_{\mathcal D_N}(e)$ is the number of
available slots per center at epoch $e$.}
\label{tab:dataset-capacity}
\small
\resizebox{\linewidth}{!}{%
\begin{tabular}{lcccc}
\toprule
Dataset & $B$ & $K'$ & $d^\star$ & $m_{\mathcal D_N}(e)$\\
\midrule
CIFAR-10 & 512 & 128 & 64 &
$5\ (e<100),\ 7\ (100\le e<250),\ 9\ (e\ge250)$\\
CIFAR-100 & 512 & 160 & 72 &
$5\ (e<125),\ 7\ (125\le e<275),\ 9\ (e\ge275)$\\
STL-10 & 512 & 128 & 48 &
$5\ (e<100),\ 7\ (100\le e<250),\ 9\ (e\ge250)$\\
Tiny ImageNet & 1024 & 320 & 96 & $5\ (e<100),\ 7\ (100\le e<250),\ 9\ (e\ge250)$\\
ImageNet-1K & 512 & 4096 & 512 & $5$\\
\bottomrule
\end{tabular}%
}
\end{table}

The two clock families used by the reported FBDM runs are
\begin{equation}
s_{\mathrm{rat},a}(t)=\frac{(1+a)t}{1+at},
\qquad
s_{\mathrm{exp},a}(t)=\frac{1-\exp(-at)}{1-\exp(-a)}.
\label{eq:reported-clocks}
\end{equation}
Table~\ref{tab:time-schedules} records the exact choice for every FBDM dataset
entry in Tables~\ref{tab:ssl-comparison} and~\ref{tab:imagenet-comparison}.
Within each dataset, all reported evaluation entries use the same time protocol.

\begin{table}[ht]
\centering
\caption{Time reparameterization for the FBDM entries in
Tables~\ref{tab:ssl-comparison} and~\ref{tab:imagenet-comparison}.
Path-uniform sampling draws
$S\sim\operatorname{Unif}([0,1])$ and sets $T=s^{-1}(S)$.}
\label{tab:time-schedules}
\small
\begin{tabular}{lccc}
\toprule
Dataset & Reported entries & Clock & Sampling convention\\
\midrule
CIFAR-10 & Linear, $5$-NN & rational, $a=0.5$ & path-uniform\\
CIFAR-100 & Linear, $5$-NN & rational, $a=0.5$ & path-uniform\\
STL-10 & Linear, $5$-NN & exponential, $a=1$ & path-uniform\\
Tiny ImageNet & Linear, $5$-NN & exponential, $a=2$ & path-uniform\\
ImageNet-1K & Linear & rational, $a=0.5$ & path-uniform\\
\bottomrule
\end{tabular}
\end{table}

\begin{table}[H]
\centering
\caption{Core configuration for the reported ResNet-18 FBDM results on
CIFAR-10, CIFAR-100, STL-10, and Tiny ImageNet.}
\label{tab:core-configuration}
\small
\begin{tabular}{ll}
\toprule
Component & Setting\\
\midrule
Backbone & ResNet-18\\
Reference geometry & Simplex-spectral\\
$K'$, $d^\star$, and capacity & See Table~\ref{tab:dataset-capacity}\\
Center perturbation & $10^{-3}$\\
Velocity time feature & raw $t$\\
Flow path & Spherical geodesic\\
FM time weighting & none ($\gamma=0$)\\
Velocity field & 8 linear layers, width 2048, ReLU\\
Alignment weight & 50\\
FM encoder grad scale & 0.75\\
Optimizer & Adam, lr $3\times10^{-5}$\\
Coupled $L^2$ & $10^{-4}$\\
Batch size & 1024 (Tiny ImageNet); 512 otherwise\\
Pretraining length & 1000 epochs\\
Augmentation & grayscale .1, Gaussian blur .25\\
\bottomrule
\end{tabular}
\end{table}

The reported ResNet-18 configuration uses RandomResizedCrop with scale
$[0.2,1]$ and aspect ratio $[0.75,4/3]$, horizontal flip probability .5,
ColorJitter $(.4,.4,.4,.1)$ applied with probability .8, grayscale probability
.1, and Gaussian blur probability .25. The two views draw independent
randomness from the same distribution except in the asymmetric-blur ablation in
Table~\ref{tab:ablations}.

\paragraph{ImageNet-1K configuration.}
We use a ResNet-50 backbone with $K'=4096$ reference components and auxiliary
representation dimension $d^\star=512$. The physical global batch size is
$B=512$. To stabilize capacitated assignment without increasing the physical
optimization batch, we aggregate up to eight consecutive physical batches and
assign them jointly over a logical window with maximum size
\[
W=8B=4096
\]
images. Capacity is enforced within this window, while parameter
updates use physical batches of $512$. Each image contributes two augmented
views, but their shared consensus is assigned once, so $W$ counts images rather
than views. We use per-component capacity $m=5$, a simplex-spectral reference
with perturbation scale $\varepsilon=10^{-3}$, and a four-layer velocity network
of width $2048$. The velocity-energy coefficient is $\beta=20$; the
flow-matching weight is $1.3$ and the explicit alignment weight is $10$, with full FM
gradient to the encoder ($\rho=1$). Training uses LARS with peak learning rate
$0.9$, a $25{,}020$-step warm-up, and a 100-epoch cosine schedule. We use
spherical geodesic paths, path-uniform sampling, and the rational clock with
$a=0.5$.
The released 100-epoch batchwise-assignment variant assigns targets within
each physical batch of $512$ images, without a cross-batch window. It uses
$\beta=25$, rational clock parameter $a=0.25$, and crop scale $[0.2,1.0]$,
and reaches $60.178\%$ linear top-1, compared with $61.41\%$ for the windowed
configuration.

\paragraph{ImageNet-1K data augmentation for the windowed configuration.}
Two views are sampled independently from the same pipeline. Each view uses a
$224\times224$ random resized crop with scale $[0.08,1.0]$ and aspect ratio
$[3/4,4/3]$, followed by horizontal flipping with probability $0.5$.
ColorJitter with brightness, contrast, saturation, and hue strengths
$(0.8,0.8,0.8,0.2)$ is applied with probability $0.8$, followed by random
grayscale conversion with probability $0.2$. Gaussian blur with a
$23\times23$ kernel and $\sigma\in[0.1,2.0]$ is applied with probability
$0.1$. The image is then converted to a tensor without ImageNet mean--standard
deviation normalization.

\FloatBarrier
\subsection{Training-Time and Memory Measurements}
\label{app:systems-protocol}

The reported CIFAR-10, CIFAR-100, STL-10, and Tiny ImageNet experiments were
each trained on a single NVIDIA V100 GPU. Each ImageNet-1K run used one of
three compute setups: a single NVIDIA H100, a single NVIDIA A100, or four
NVIDIA V100 GPUs in parallel.
The DM--FBDM comparison uses one Tesla V100-SXM2-16GB, FP32, and batch size
512. DM's eight-layer critic and FBDM's eight-layer velocity field differ in parameter
count by less than $0.025\%$ (Table~\ref{tab:auxiliary-parameters}); both
methods retain their native projection heads. DM takes five critic updates
with gradient penalty per encoder update. We average the times of complete
training epochs 2 and 3 after one warm-up epoch. Speedups are the ratios of
DM to FBDM epoch times.
FBDM uses about 0.21~GiB more peak allocated memory and 0.20--0.37~GiB more
reserved memory in these three comparisons (Table~\ref{tab:systems}).

\begin{table}[!htb]
\centering
\caption{Reported V100 training time (seconds per epoch) for the methods in
Table~\ref{tab:ssl-comparison}; these are the values plotted in
Figure~\ref{fig:systems-all}.}
\label{tab:systems-all-values}
\small
\setlength{\tabcolsep}{5pt}
\begin{tabular}{lrrrr}
\toprule
Method & CIFAR-10 & CIFAR-100 & STL-10 & Tiny ImageNet\\
\midrule
FBDM & $26.72$ & $26.97$ & $66.42$ & $61.67$\\
SimCLR & $27.51$ & $27.17$ & $66.89$ & $63.31$\\
BYOL & $30.94$ & $30.50$ & $73.40$ & $69.76$\\
SimSiam & $25.80$ & $25.56$ & $64.87$ & $62.22$\\
Barlow Twins & $26.87$ & $27.54$ & $65.82$ & $62.10$\\
VICReg & $27.09$ & $28.81$ & $66.04$ & $63.05$\\
Zero-ICL & $29.41$ & $28.98$ & $67.76$ & $65.16$\\
W-MSE & $25.92$ & $26.40$ & $64.85$ & $62.88$\\
CW-RGP & $80.49$ & $37.52$ & $85.72$ & $80.16$\\
Shuffled-DBN & $25.82$ & $26.52$ & $64.51$ & $62.74$\\
DM & $63.05$ & $62.62$ & $144.03$ & --\\
\bottomrule
\end{tabular}
\end{table}

\FloatBarrier

\paragraph{ImageNet-1K training time.}
We measure FBDM and SimCLR over two complete pretraining epochs using four
NVIDIA V100 GPUs, a ResNet-50 backbone, global batch size 512, and automatic
mixed precision (AMP). Their per-epoch times are close
(Table~\ref{tab:imagenet-comparison}).

For the matched DM--FBDM comparison, Table~\ref{tab:memory} reports peak CUDA
allocation and reservation for the complete training implementation, taking
the maximum of the epochwise peaks over epochs 2 and 3 after one warm-up epoch.

\begin{table}[ht]
\centering
\caption{Auxiliary parameter counts (millions) in the paired comparison.
The relative gap uses the DM critic count as its denominator.}
\label{tab:auxiliary-parameters}
\small
\begin{tabular}{lrrr}
\toprule
Dataset & DM critic & FBDM velocity field & Gap (\%)\\
\midrule
CIFAR-10 & $25.450613$ & $25.444416$ & $0.02435$\\
CIFAR-100 & $25.475337$ & $25.477192$ & $0.00728$\\
STL-10 & $25.376513$ & $25.378864$ & $0.00926$\\
\bottomrule
\end{tabular}
\end{table}
\FloatBarrier

%% file: appendices/ablations.tex
\section{Additional Experiments and Ablations}
\label{app:ablations}

\subsection{Direct-Target Regression without Flow Matching}
\label{app:direct-regression}
To test whether direct matching to the assigned reference targets explains the
ImageNet-1K result, we disable the velocity field and compare two regression
geometries. Using the normalized representations $\mathbf z_j^{(n)}$ from
\eqref{eq:representation} and their shared target $\mathbf r^{(n)}$ from
\eqref{eq:target}, define
\begin{align}
\widehat{\mathcal R}_{\mathrm{euc},B}
&=\frac{1}{2Bd^\star}\sum_{n=1}^{B}\sum_{j=1}^{2}
\|\mathbf z_j^{(n)}-\mathbf r^{(n)}\|_2^2,
\label{eq:direct-euclidean}\\
\widehat{\mathcal R}_{\mathrm{geo},B}
&=\frac{1}{2Bd^\star}\sum_{n=1}^{B}\sum_{j=1}^{2}
\arccos^2\!\left((\mathbf z_j^{(n)})^\top\mathbf r^{(n)}\right).
\label{eq:direct-geodesic}
\end{align}
The first uses Euclidean distance on the sphere; the second uses the
geodesic angle. Both controls retain the explicit two-view alignment term
\eqref{eq:alignment} and otherwise use the same assignment, encoder and
projector, training procedure, and frozen-backbone evaluation as the reported
ImageNet-1K run. Neither trains a velocity field or uses its energy penalty.

Geodesic regression improves on Euclidean regression by $0.73$ percentage points,
showing that the choice of direct-regression geometry matters. Both direct
controls remain below the complete velocity-mediated FBDM result under this
evaluation protocol (Table~\ref{tab:direct-regression-main}).

\subsection{ImageNet-1K Initialization and Velocity-Energy Regularization}
\label{app:imagenet-energy}
We isolate the effect of velocity-energy regularization on ImageNet-1K using
three 50-epoch runs that differ only in $\beta$. They share the training
configuration and frozen-backbone linear-evaluation protocol in
Appendix~\ref{app:configs}.

\subsubsection{Velocity-Energy Ablation}
\label{app:velocity-energy-ablation}
\begin{table}[H]
\centering
\caption{Velocity-energy ablation on ImageNet-1K after 50 epochs of pretraining,
with a common frozen-backbone linear-evaluation protocol.}
\label{tab:velocity-energy-ablation}
\small
\begin{tabular}{cc}
\toprule
Energy coefficient $\beta$ & Linear top-1 (\%)\\
\midrule
$0$ & $2.044$\\
$10$ & $52.765$\\
$20$ & $54.763$\\
\bottomrule
\end{tabular}
\end{table}

Table~\ref{tab:velocity-energy-ablation} shows that adding the energy penalty
substantially improves linear accuracy in this configuration. Our motivation
is that an unconstrained auxiliary field may absorb much of the transport to
the assigned targets, leaving less pressure for the encoder representation to
adapt. Penalizing velocity magnitude encourages the representation to share
more of this adaptation. To understand when this may matter, we next examine
the geometry of the representations at initialization.

\subsubsection{Projected Geometry at Random Initialization}
\label{app:initial-rank}
To characterize geometry before training, we use spectral-entropy effective
rank, which summarizes how evenly variance is distributed across covariance
directions. We measure the centered-covariance effective rank of the raw
projector outputs $\mathbf y_j^{(n)}=h_\psi(f_\theta(\mathbf x_j^{(n)}))$,
their normalized representations $\mathbf z_j^{(n)}$ from
\eqref{eq:representation}, and the assignment consensus $\mathbf q^{(n)}$
from \eqref{eq:consensus}. For $n$ vectors
$\mathbf a_i\in\mathbb R^{d^\star}$, define
\begin{gather}
\overline{\mathbf a}=\frac1n\sum_{i=1}^n\mathbf a_i,
\qquad
\bm\Sigma=\frac1n\sum_{i=1}^n
(\mathbf a_i-\overline{\mathbf a})(\mathbf a_i-\overline{\mathbf a})^\top,
\notag\\
p_\ell=\frac{\lambda_\ell}{\sum_{j=1}^{d^\star}\lambda_j},
\qquad
r_{\mathrm{eff}}=\exp\!\left(-\sum_{\ell=1}^{d^\star}p_\ell\log p_\ell\right),
\label{eq:effective-rank}
\end{gather}
where $\lambda_\ell$ are the eigenvalues of $\bm\Sigma$, its trace is positive,
the logarithm is natural, and $0\log0=0$.
For each dataset, we use 8192 original images in two disjoint windows of 4096
images at random initialization, before any parameter update. Each image
provides two augmented views: each window contains 8192 vectors for
$\mathbf y$ and $\mathbf z$, and 4096 consensus vectors for $\mathbf q$.
We center and compute effective rank separately in each window, then average
the two ranks. The images come from the training split, except STL-10, which
uses \texttt{train+unlabeled}.
Table~\ref{tab:initial-rank} gives the random-initialization diagnostics at the
listed matching dimensions. The consensus effective-rank fraction is
$72.84$--$81.05\%$ for the four smaller datasets and $37.94\%$ for ImageNet.
To measure direction concentration, set
$\overline{\mathbf q}=n^{-1}\sum_{i=1}^n\mathbf q^{(i)}$ and
$R_q=\|\overline{\mathbf q}\|_2$ in each window, then average the statistics.
For ImageNet, $R_q=0.7088$, and the median angle to the mean direction
$\overline{\mathbf q}/\|\overline{\mathbf q}\|_2$ is $43.81^\circ$.
The leading covariance axis accounts for $12.46\%$ of the variance.
To connect these geometric diagnostics to assignment, we also measure center
use. Within each 4096-image window, nearest-center matching selects 441 of the
4096 centers on average, with the most popular $1\%$ receiving $64.72\%$ of
images; capacity-constrained matching with diagnostic capacity $m=5$
selects 1216 centers on average.
The lower effective-rank fraction and directional concentration indicate a
more anisotropic, cone-like starting geometry on ImageNet. This observation
motivates the hypothesis that, in this setting, an unconstrained velocity field
may bear too much of the transport rather than encouraging adaptation of the
encoder.

\begin{table}[H]
\centering
\caption{Centered-covariance spectral-entropy effective rank at random
initialization. The last column is $100r_{\mathrm{eff}}(\mathbf q)/d^\star$.}
\label{tab:initial-rank}
\small
\begin{tabular}{lrrrrr}
\toprule
Dataset & $d^\star$ & $\mathbf y$ & $\mathbf z$ & $\mathbf q$ & $\mathbf q$ rank fraction (\%)\\
\midrule
CIFAR-10 & 64 & 32.56 & 50.24 & 48.44 & 75.68\\
CIFAR-100 & 72 & 34.81 & 56.54 & 53.67 & 74.54\\
STL-10 & 48 & 26.84 & 40.07 & 38.90 & 81.05\\
Tiny ImageNet & 96 & 48.82 & 73.44 & 69.92 & 72.84\\
ImageNet-1K & 512 & 37.81 & 203.52 & 194.27 & 37.94\\
\bottomrule
\end{tabular}
\end{table}

Proposition~\ref{prop:velocity-energy-matching} in Appendix~\ref{app:grad-scale}
provides a complementary distributional perspective: under its stated
conditions, the mismatch between the initial representation distribution and
the assigned-target distribution is controlled by the flow-matching risk and
the field's velocity energy. This connects the penalty to adaptation of the
starting representations.

\subsection{Spherical versus Linear Transport Paths}
\label{app:linear-interpolation}
Table~\ref{tab:linear-path-ablation} compares the spherical-path FBDM results
reported in Table~\ref{tab:ssl-comparison} with a linear transport path in place
of \eqref{eq:slerp},
$\bm\gamma^{\mathrm{lin}}_{\mathbf z,\mathbf r}(s)
=(1-s)\mathbf z+s\mathbf r$.

\begin{table}[H]
\centering
\caption{Reported FBDM top-1 accuracy (\%) with the spherical path used in
Table~\ref{tab:ssl-comparison} and with direct linear interpolation. Linear
denotes frozen-representation linear evaluation; $5$-NN uses $k=5$.}
\label{tab:linear-path-ablation}
\small
\begin{tabular}{lcccc}
\toprule
& \multicolumn{2}{c}{Spherical path}
& \multicolumn{2}{c}{Linear path}\\
\cmidrule(lr){2-3}\cmidrule(lr){4-5}
Dataset & Linear & $5$-NN & Linear & $5$-NN\\
\midrule
CIFAR-10 & $92.37$ & $89.68$ & $90.50$ & $88.47$\\
CIFAR-100 & $66.59$ & $56.74$ & $65.01$ & $55.49$\\
STL-10 & $89.79$ & $86.23$ & $87.94$ & $83.85$\\
Tiny ImageNet & $48.55$ & $33.01$ & $47.47$ & $32.65$\\
\bottomrule
\end{tabular}
\end{table}

\subsection{Alignment and Optimization Diagnostics}
\label{app:alignment-optimization}
We test whether shared-target flow matching can replace explicit view
alignment. In a matched 1000-epoch diagnostic that retained the complete FM
branch but changed only $\lambda$ from $50$ to $0$, the final online
linear/$5$-NN accuracies were $80.85\%/75.02\%$ and never exceeded
$80.86\%/75.19\%$. In this diagnostic, the shared endpoint does not substitute
for explicit alignment of the starting representations. Moving explicit
alignment to the transported representations gave online peaks of
$78.96\%/72.80\%$.
These accuracies are online measurements made during pretraining; the paired
comparisons below use the fixed protocol in Appendix~\ref{app:configs}.

Table~\ref{tab:ablations} gives matched fixed-evaluation comparisons on
CIFAR-10. Scaling the FM gradient reaching the encoder by a factor of
0.75 and removing the 500-step optimizer warmup improve both metrics. Asymmetric
two-view blur mainly improves linear accuracy. Time weighting multiplies each
per-path squared FM error by
$[s'(t)]^{2\gamma}/\mathbb E_{T\sim q_T}\{[s'(T)]^{2\gamma}\}$;
$\gamma=0.5$ mainly benefits $5$-NN.

\begin{table}[H]
\centering
\caption{Matched CIFAR-10 ablations evaluated at the 1000-epoch pretraining
checkpoint using the fixed protocol in Appendix~\ref{app:configs}. Each row
compares a variant with its matched control; $\Delta$ is variant minus control,
in percentage points.}
\label{tab:ablations}
\small
\setlength{\tabcolsep}{3.9pt}
\begin{tabular}{llcc}
\toprule
Factor & Variant / control & $\Delta$ Linear & $\Delta$ 5-NN\\
\midrule
FM encoder gradient & $0.75/1.00$ & $+0.860$ & $+0.57$\\
LR warmup steps & $0/500$ & $+0.330$ & $+0.56$\\
View blur $(p_1,p_2)$ & $(.20,.30)/(.25,.25)$ & $+0.420$ & $+0.08$\\
Time-loss exponent & $0.5/0$ & $+0.033$ & $+0.89$\\
\bottomrule
\end{tabular}
\end{table}

Other optimization, augmentation, assignment, flow, sampling, and center
variants gave no stable fixed improvement.

%% file: appendices/dm_diagnostics.tex
\subsection{Historical DM and FBDM Diagnostics}
\label{app:dm-diagnostics}

These completed critic-based DM runs study optimization during pretraining
under the configuration below. Table~\ref{tab:ssl-comparison} reports the
published DM benchmark scores, and Section~\ref{sec:experiments} reports the
matched-cost comparison. The diagnostic runs use CIFAR-10, a randomly initialized CIFAR-stem ResNet-18, seed 0,
a two-layer projector (hidden width 1000, output dimension 384), and a fixed
ETF-inspired reference with $K'=d^\star=384$ and noise $\varepsilon=0.001$.
Training uses the critic-based DM objective with this ETF-inspired reference.
Each run has 1000 pretraining epochs at physical batch size 512, with 97 full
minibatches per epoch. The encoder and projector use coupled-decay Adam with
$(\beta_1,\beta_2)=(0.9,0.99)$ and a 500-update linear learning-rate warmup;
the target rate is reached during epoch 6 and remains constant thereafter.
For the two plotted runs, the critic uses Adam with the same betas, learning
rate $10^{-3}$ and weight decay $10^{-4}$, with five updates per encoder update.

\paragraph{What the curves measure.}
The recorded DM encoder/projector minibatch objective is
\begin{equation}
L_{\mathrm{enc}}
=\frac1B\sum_{i=1}^{B}\|\mathbf u_i-\mathbf v_i\|_2^2
+\frac1{2B}\sum_{\ell=1}^{2B}g(\mathbf r_\ell)
-\frac1{2B}\sum_{i=1}^{B}
 \bigl[g(\mathbf u_i)+g(\mathbf v_i)\bigr].
\label{eq:dm-diagnostic-loss}
\end{equation}
Here $L_{\mathrm{enc}}$ denotes the encoder/projector minibatch objective,
$B=512$ is the number of images, and $d^\star=384$ is the projector dimension.
For image $i\in\{1,\ldots,B\}$, $\mathbf u_i,\mathbf v_i\in\mathbb S^{d^\star-1}$
are the normalized projector outputs of its two views.
The vectors $\mathbf r_\ell\in\mathbb R^{d^\star}$, $\ell=1,\ldots,2B$,
form the sampled reference batch, and
$g:\mathbb R^{d^\star}\to\mathbb R$ is the current scalar-valued critic.
Both coefficients are one. Each plotted loss is the mean of the 97 minibatch
objectives, evaluated before each encoder update with the then-current critic.
The critic separately minimizes the negative score gap plus a unit-weight
gradient penalty. The plotted scalar contains exactly the alignment and current
critic score gap in \eqref{eq:dm-diagnostic-loss}; Adam weight decay acts through
the optimizer updates. Each epoch mean spans changing encoder and critic states.

Every 20 epochs, a fresh linear classifier is fitted for 500 epochs on all
50,000 training images using frozen raw backbone features, then evaluated on
the 10,000 official CIFAR-10 test images. The author-selected FBDM reference
reports test accuracy on the same basic frozen-backbone evaluation protocol;
its training objective combines velocity regression and alignment.
Figure~\ref{fig:dm-diagnostics} shows the two DM trajectories alongside this
single FBDM reference.

\begin{figure}[!htbp]
\centering
\includegraphics[width=\linewidth]{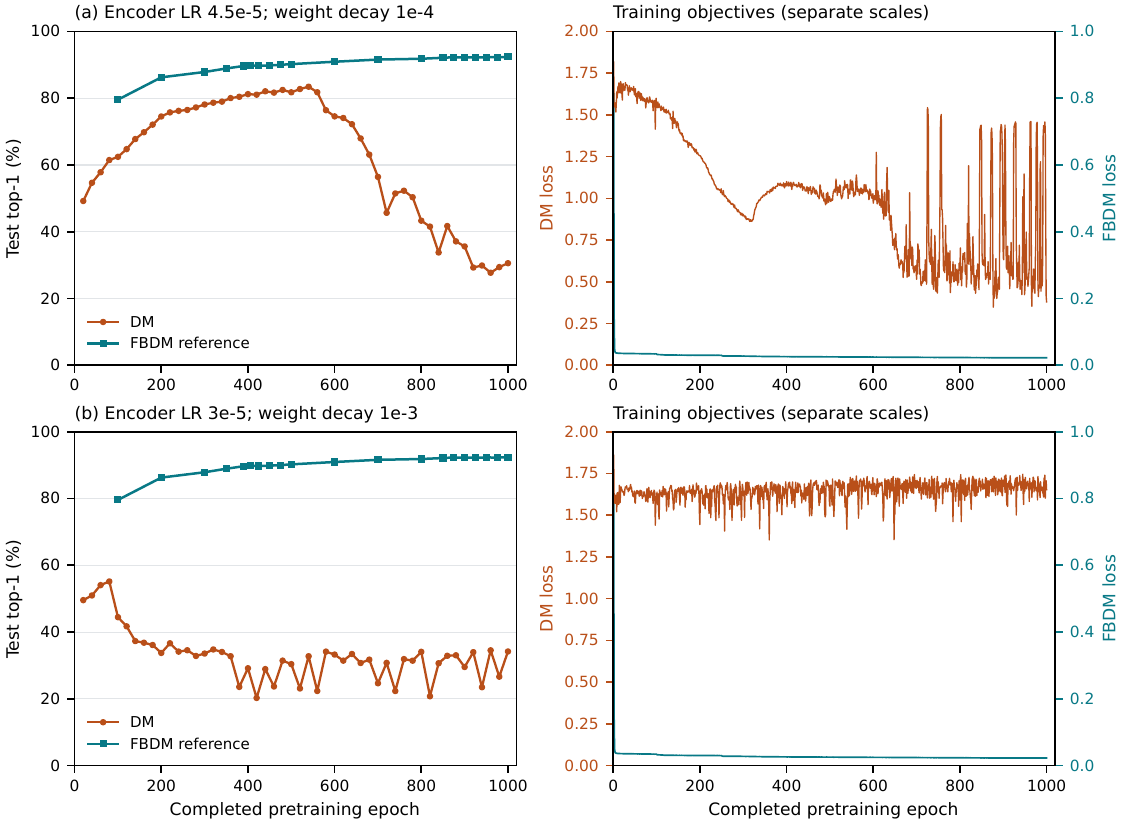}
\caption{\textbf{Historical CIFAR-10 trajectories.}
Changing a single encoder/projector hyperparameter from the reference:
(a) LR $3\times10^{-5}\to4.5\times10^{-5}$ ($1.5\times$), with weight decay
$10^{-4}$ fixed; (b) weight decay $10^{-4}\to10^{-3}$ ($10\times$), with
LR $3\times10^{-5}$ fixed.
Left panels show frozen-backbone linear test accuracy; right panels show the
DM and FBDM training objectives on separate labeled scales.}
\label{fig:dm-diagnostics}
\end{figure}

\paragraph{Loss behavior across the two trials.}
Figure~\ref{fig:dm-diagnostics} shows distinct training-loss trajectories
accompanying the accuracy declines. In the learning-rate trial, the target
rate is $1.5\times$ its reference value of $3\times10^{-5}$; test accuracy
peaks at $83.44\%$ at epoch 540 and ends at $30.59\%$.
The final 20 losses range from $0.3801$ to $1.4583$, with a final value below
the loss at the accuracy peak; the trajectory is not monotone.
In the weight-decay trial, coupled decay is tenfold its reference value of
$10^{-4}$; test accuracy peaks at $55.15\%$ at epoch 80 and ends at $34.18\%$.
Loss remains roughly level to rising during this decline: its epoch-61--80
mean is $1.6222$, versus $1.6642$ over epochs 981--1000.

%% file: appendices/related_work.tex
\section{Extended Related Work}
\label{app:related}

\paragraph{Collapse prevention in SSL.}
Contrastive learning uses negative pairs and uniformity pressure to retain
information across the representation space
\citep{chen2020simple,wang2020understanding}. Non-contrastive methods rely on
predictor asymmetry, target networks, or stop-gradient mechanisms
\citep{grill2020bootstrap,chen2021exploring}. Redundancy-reduction and whitening
methods instead constrain second-order statistics
\citep{zbontar2021barlow,ermolov2021whitening,bardes2022vicreg,weng2022investigation}.
Other approaches use swapped cluster assignments \citep{caron2020swav},
self-distillation \citep{caron2021dino}, or orthogonality regularization
\citep{he2024orthogonality}.
Recent work develops these SSL objectives and their analysis
\citep{huang2024ldreg,duan2025advssl,simeoni2025dinov3,
zhang2026structural,luthra2026alignment}.

\paragraph{Explicit distribution matching.}
DM prescribes a reference distribution and matches the encoded distribution
to it through a critic-based Wasserstein objective \citep{jiao2026distribution}.
FBDM retains an explicit reference but regresses analytic velocities along
assigned representation--reference paths.

\paragraph{Flow matching and geometric transport.}
Neural ODEs describe continuous transformations through learned vector fields
\citep{chen2018neural}. Flow matching avoids simulation during training by
regressing analytically available conditional velocities
\citep{lipman2023flow}. Riemannian flow matching extends this construction to
manifolds through geodesic conditional paths and gives closed-form targets on
simple geometries \citep{chen2024geometries}. We use its spherical great-circle
construction so both interpolants and target velocities respect representation
geometry. REPA uses pretrained SSL features to guide generative diffusion
representations; its analysis identifies useful intermediate features and
a later-layer linear-probe decline associated with high-frequency generation
details \citep{yu2025repa}. FBDM instead learns encoder-induced source
representations through reference-directed flow, without image reconstruction
or generation; the auxiliary field is discarded downstream.

\paragraph{Theory of representation and flow learning.}
SSL theory connects representation geometry to transfer, generalization and
robustness \citep{saunshi2019theoretical,wang2020understanding,haochen2021provable,
haochen2022beyond,huang2023towards,duan2025advssl,jiao2026distribution}.
Flow-learning theory controls transport error and the statistical and numerical
behavior of learned flows \citep{fukumizu2025flow,gao2025generalpfode,huang2025pfode,
cai2025minimaxpfode,li2024sharppfode,tang2025adaptivepfode}.
Here we study how auxiliary velocity regression controls classification on
the initial representation under stated conditions.

\paragraph{Simplex and frame geometry.}
Simplex equiangular tight frames appear in the neural-collapse literature as
maximally separated class means under suitable dimensionality
\citep{papyan2020prevalence}. The orthogonal reference construction used in DM
requires $K'\le d^\star$. FBDM instead derives a label-free finite prior from a
low-rank approximation of the ideal simplex Gram matrix, extending explicit
reference design to the practically useful regime $K'>d^\star$.

%% file: references.bib
@inproceedings{chen2020simple,
  title={A Simple Framework for Contrastive Learning of Visual Representations},
  author={Chen, Ting and Kornblith, Simon and Norouzi, Mohammad and Hinton, Geoffrey},
  booktitle={International Conference on Machine Learning},
  year={2020}
}

@inproceedings{he2020momentum,
  title={Momentum Contrast for Unsupervised Visual Representation Learning},
  author={He, Kaiming and Fan, Haoqi and Wu, Yuxin and Xie, Saining and Girshick, Ross},
  booktitle={Proceedings of the IEEE/CVF Conference on Computer Vision and Pattern Recognition},
  pages={9729--9738},
  year={2020},
  url={https://openaccess.thecvf.com/content_CVPR_2020/html/He_Momentum_Contrast_for_Unsupervised_Visual_Representation_Learning_CVPR_2020_paper.html}
}

@inproceedings{grill2020bootstrap,
  title={Bootstrap Your Own Latent: A New Approach to Self-Supervised Learning},
  author={Grill, Jean-Bastien and Strub, Florian and Altch{\'e}, Florent and Tallec, Corentin and Richemond, Pierre H. and Buchatskaya, Elena and Doersch, Carl and Avila Pires, Bernardo and Guo, Zhaohan Daniel and Gheshlaghi Azar, Mohammad and Piot, Bilal and Kavukcuoglu, Koray and Munos, R{\'e}mi and Valko, Michal},
  booktitle={Advances in Neural Information Processing Systems},
  year={2020}
}

@inproceedings{caron2020swav,
  title={Unsupervised Learning of Visual Features by Contrasting Cluster Assignments},
  author={Caron, Mathilde and Misra, Ishan and Mairal, Julien and Goyal, Priya and Bojanowski, Piotr and Joulin, Armand},
  booktitle={Advances in Neural Information Processing Systems},
  volume={33},
  year={2020},
  url={https://proceedings.neurips.cc/paper_files/paper/2020/hash/70feb62b69f16e0238f741fab228fec2-Abstract.html}
}

@inproceedings{chen2021exploring,
  title={Exploring Simple Siamese Representation Learning},
  author={Chen, Xinlei and He, Kaiming},
  booktitle={IEEE/CVF Conference on Computer Vision and Pattern Recognition},
  year={2021}
}

@inproceedings{caron2021dino,
  title={Emerging Properties in Self-Supervised Vision Transformers},
  author={Caron, Mathilde and Touvron, Hugo and Misra, Ishan and J{\'e}gou, Herv{\'e} and Mairal, Julien and Bojanowski, Piotr and Joulin, Armand},
  booktitle={Proceedings of the IEEE/CVF International Conference on Computer Vision},
  pages={9650--9660},
  year={2021},
  url={https://openaccess.thecvf.com/content/ICCV2021/html/Caron_Emerging_Properties_in_Self-Supervised_Vision_Transformers_ICCV_2021_paper.html}
}

@inproceedings{zbontar2021barlow,
  title={Barlow Twins: Self-Supervised Learning via Redundancy Reduction},
  author={Zbontar, Jure and Jing, Li and Misra, Ishan and LeCun, Yann and Deny, St{\'e}phane},
  booktitle={International Conference on Machine Learning},
  year={2021}
}

@inproceedings{bardes2022vicreg,
  title={{VICReg}: Variance-Invariance-Covariance Regularization for Self-Supervised Learning},
  author={Bardes, Adrien and Ponce, Jean and LeCun, Yann},
  booktitle={International Conference on Learning Representations},
  year={2022}
}

@inproceedings{ermolov2021whitening,
  title={Whitening for Self-Supervised Representation Learning},
  author={Ermolov, Aleksandr and Siarohin, Aliaksandr and Sangineto, Enver and Sebe, Nicu},
  booktitle={International Conference on Machine Learning},
  year={2021}
}

@inproceedings{weng2022investigation,
  title={An Investigation into Whitening Loss for Self-Supervised Learning},
  author={Weng, Xi and Huang, Lei and Zhao, Lei and Anwer, Rao Muhammad and Khan, Salman and Khan, Fahad Shahbaz},
  booktitle={Advances in Neural Information Processing Systems},
  volume={35},
  year={2022},
  doi={10.52202/068431-2157},
  url={https://proceedings.neurips.cc/paper_files/paper/2022/hash/c057cb81b8d3c67093427bf1c16a4e9f-Abstract-Conference.html}
}

@inproceedings{yu2025repa,
  title={Representation Alignment for Generation: Training Diffusion Transformers Is Easier Than You Think},
  author={Yu, Sihyun and Kwak, Sangkyung and Jang, Huiwon and Jeong, Jongheon and Huang, Jonathan and Shin, Jinwoo and Xie, Saining},
  booktitle={International Conference on Learning Representations},
  year={2025},
  url={https://proceedings.iclr.cc/paper_files/paper/2025/hash/d9e42b4d7163931f3689d6d6fbaa11d0-Abstract-Conference.html}
}

@misc{chefer2026selfflow,
  title={Self-Supervised Flow Matching for Scalable Multi-Modal Synthesis},
  author={Chefer, Hila and Esser, Patrick and Lorenz, Dominik and Podell, Dustin and Raja, Vikash and Tong, Vinh and Torralba, Antonio and Rombach, Robin},
  year={2026},
  eprint={2603.06507},
  archivePrefix={arXiv},
  url={https://arxiv.org/abs/2603.06507}
}

@inproceedings{ukita2026senflow,
  title={High-Performance Self-Supervised Learning by Joint Training of Flow Matching},
  author={Ukita, Kosuke and Okita, Tsuyoshi},
  booktitle={Proceedings of the 29th International Conference on Artificial Intelligence and Statistics},
  series={Proceedings of Machine Learning Research},
  volume={300},
  pages={4492--4500},
  year={2026},
  publisher={PMLR},
  url={https://proceedings.mlr.press/v300/ukita26a.html}
}

@inproceedings{huang2023towards,
  title={Towards the Generalization of Contrastive Self-Supervised Learning},
  author={Huang, Weiran and Yi, Mingyang and Zhao, Xuyang and Jiang, Zihao},
  booktitle={International Conference on Learning Representations},
  year={2023},
  url={https://openreview.net/forum?id=XDJwuEYHhme}
}

@inproceedings{huang2024ldreg,
  title={{LDReg}: Local Dimensionality Regularized Self-Supervised Learning},
  author={Huang, Hanxun and Campello, Ricardo J. G. B. and Erfani, Sarah Monazam and Ma, Xingjun and Houle, Michael E. and Bailey, James},
  booktitle={International Conference on Learning Representations},
  year={2024},
  url={https://openreview.net/forum?id=oZyAqjAjJW}
}

@inproceedings{he2024orthogonality,
  title={Preventing Dimensional Collapse in Self-Supervised Learning via Orthogonality Regularization},
  author={He, Junlin and Du, Jinxiao and Ma, Wei},
  booktitle={Advances in Neural Information Processing Systems},
  volume={37},
  year={2024},
  doi={10.52202/079017-3028},
  url={https://proceedings.neurips.cc/paper_files/paper/2024/hash/ad7922fd4650f8aba5d8b067e622ca84-Abstract-Conference.html}
}

@inproceedings{duan2025advssl,
  title={{Adv-SSL}: Adversarial Self-Supervised Representation Learning with Theoretical Guarantees},
  author={Duan, Chenguang and Jiao, Yuling and Lin, Huazhen and Ma, Wensen and Yang, Jerry},
  booktitle={Advances in Neural Information Processing Systems},
  volume={38},
  year={2025},
  doi={10.52202/085713-4097},
  url={https://proceedings.neurips.cc/paper_files/paper/2025/hash/b216ce49d808157ba14adec9453983a4-Abstract-Conference.html}
}

@misc{simeoni2025dinov3,
  title={{DINOv3}},
  author={Sim{\'e}oni, Oriane and others},
  year={2025},
  eprint={2508.10104},
  archivePrefix={arXiv},
  primaryClass={cs.CV},
  url={https://arxiv.org/abs/2508.10104}
}

@inproceedings{zhang2026structural,
  title={Self-Supervised Learning from Structural Invariance},
  author={Zhang, Yipeng and Ghaemi, Hafez and Lee, Jungyoon and Bakhtiari, Shahab and Muller, Eilif B. and Charlin, Laurent},
  booktitle={International Conference on Learning Representations},
  year={2026},
  url={https://proceedings.iclr.cc/paper_files/paper/2026/hash/d26d1923ce22d5afc3ee0db15ef4d900-Abstract-Conference.html}
}

@inproceedings{luthra2026alignment,
  title={On the Alignment Between Supervised and Self-Supervised Contrastive Learning},
  author={Luthra, Achleshwar and Mishra, Priyadarsi and Galanti, Tomer},
  booktitle={International Conference on Learning Representations},
  year={2026},
  url={https://proceedings.iclr.cc/paper_files/paper/2026/hash/fa76985f05e0a25c66528308dda33de0-Abstract-Conference.html}
}

@inproceedings{saunshi2019theoretical,
  title={A Theoretical Analysis of Contrastive Unsupervised Representation Learning},
  author={Saunshi, Nikunj and Plevrakis, Orestis and Arora, Sanjeev and Khodak, Mikhail and Khandeparkar, Hrishikesh},
  booktitle={Proceedings of the 36th International Conference on Machine Learning},
  series={Proceedings of Machine Learning Research},
  volume={97},
  pages={5628--5637},
  year={2019},
  url={https://proceedings.mlr.press/v97/saunshi19a.html}
}

@inproceedings{haochen2021provable,
  title={Provable Guarantees for Self-Supervised Deep Learning with Spectral Contrastive Loss},
  author={HaoChen, Jeff Z. and Wei, Colin and Gaidon, Adrien and Ma, Tengyu},
  booktitle={Advances in Neural Information Processing Systems},
  volume={34},
  pages={5000--5011},
  year={2021},
  url={https://proceedings.neurips.cc/paper_files/paper/2021/hash/27debb435021eb68b3965290b5e24c49-Abstract.html}
}

@inproceedings{haochen2022beyond,
  title={Beyond Separability: Analyzing the Linear Transferability of Contrastive Representations to Related Subpopulations},
  author={HaoChen, Jeff Z. and Wei, Colin and Kumar, Ananya and Ma, Tengyu},
  booktitle={Advances in Neural Information Processing Systems},
  volume={35},
  pages={26889--26902},
  year={2022},
  url={https://proceedings.neurips.cc/paper_files/paper/2022/hash/ac112e8ffc4e5b9ece32070440a8ca43-Abstract-Conference.html}
}

@inproceedings{hua2021feature,
  title={On Feature Decorrelation in Self-Supervised Learning},
  author={Hua, Tianyu and Wang, Wenxiao and Xue, Zihui and Ren, Sucheng and Wang, Yue and Zhao, Hang},
  booktitle={Proceedings of the IEEE/CVF International Conference on Computer Vision},
  pages={9598--9608},
  year={2021}
}

@inproceedings{zhang2022zerocl,
  title={Zero-{CL}: Instance and Feature Decorrelation for Negative-Free Symmetric Contrastive Learning},
  author={Zhang, Shaofeng and Zhu, Feng and Yan, Junchi and Zhao, Rui and Yang, Xiaokang},
  booktitle={International Conference on Learning Representations},
  year={2022},
  url={https://openreview.net/forum?id=RAW9tCdVxLj}
}

@misc{jiao2026distribution,
      title={Bringing Generative Learning to Representation Learning: Self-Supervised Transfer Learning as Distribution Matching}, 
      author={Yuling Jiao and Wensen Ma and Defeng Sun and Hansheng Wang and Yang Wang},
      year={2026},
      eprint={2502.14424},
      archivePrefix={arXiv},
      primaryClass={stat.ML},
      url={https://arxiv.org/abs/2502.14424}, 
}

@inproceedings{lipman2023flow,
  title={Flow Matching for Generative Modeling},
  author={Lipman, Yaron and Chen, Ricky T. Q. and Ben-Hamu, Heli and Nickel, Maximilian and Le, Matt},
  booktitle={International Conference on Learning Representations},
  year={2023}
}

@inproceedings{chen2024geometries,
  title={Flow Matching on General Geometries},
  author={Chen, Ricky T. Q. and Lipman, Yaron},
  booktitle={International Conference on Learning Representations},
  year={2024},
  url={https://openreview.net/forum?id=g7ohDlTITL}
}

@inproceedings{liu2023rectified,
  title={Flow Straight and Fast: Learning to Generate and Transfer Data with Rectified Flow},
  author={Liu, Xingchao and Gong, Chengyue and Liu, Qiang},
  booktitle={International Conference on Learning Representations},
  year={2023},
  url={https://openreview.net/forum?id=XVjTT1nw5z}
}

@inproceedings{fukumizu2025flow,
  title={Flow Matching Achieves Almost Minimax Optimal Convergence},
  author={Fukumizu, Kenji and Suzuki, Taiji and Isobe, Noboru and Oko, Kazusato and Koyama, Masanori},
  booktitle={International Conference on Learning Representations},
  year={2025},
  eprint={2405.20879},
  archivePrefix={arXiv},
  primaryClass={cs.LG},
  url={https://proceedings.iclr.cc/paper_files/paper/2025/hash/44a427b77e9727fcec560e2f8d6925e0-Abstract-Conference.html}
}

@inproceedings{gao2025generalpfode,
  title={Convergence Analysis for General Probability Flow {ODE}s of Diffusion Models in Wasserstein Distances},
  author={Gao, Xuefeng and Zhu, Lingjiong},
  booktitle={Proceedings of the 28th International Conference on Artificial Intelligence and Statistics},
  series={Proceedings of Machine Learning Research},
  volume={258},
  pages={1009--1017},
  year={2025},
  eprint={2401.17958},
  archivePrefix={arXiv},
  primaryClass={stat.ML},
  url={https://proceedings.mlr.press/v258/gao25c.html}
}

@article{huang2025pfode,
  title={Convergence Analysis of Probability Flow {ODE} for Score-Based Generative Models},
  author={Huang, Daniel Zhengyu and Huang, Jiaoyang and Lin, Zhengjiang},
  journal={IEEE Transactions on Information Theory},
  volume={71},
  number={6},
  pages={4581--4601},
  year={2025},
  doi={10.1109/TIT.2025.3557050},
  url={https://doi.org/10.1109/TIT.2025.3557050}
}

@misc{cai2025minimaxpfode,
  title={Minimax Optimality of the Probability Flow {ODE} for Diffusion Models},
  author={Cai, Changxiao and Li, Gen},
  year={2025},
  eprint={2503.09583},
  archivePrefix={arXiv},
  primaryClass={cs.LG},
  url={https://arxiv.org/abs/2503.09583}
}

@misc{li2024sharppfode,
  title={A Sharp Convergence Theory for the Probability Flow {ODE}s of Diffusion Models},
  author={Li, Gen and Wei, Yuting and Chi, Yuejie and Chen, Yuxin},
  year={2024},
  eprint={2408.02320},
  archivePrefix={arXiv},
  primaryClass={cs.LG},
  url={https://arxiv.org/abs/2408.02320}
}

@misc{tang2025adaptivepfode,
  title={Adaptivity and Convergence of Probability Flow {ODE}s in Diffusion Generative Models},
  author={Tang, Jiaqi and Yan, Yuling},
  year={2025},
  eprint={2501.18863},
  archivePrefix={arXiv},
  primaryClass={stat.ML},
  url={https://arxiv.org/abs/2501.18863}
}

@inproceedings{arjovsky2017wasserstein,
  title={Wasserstein Generative Adversarial Networks},
  author={Arjovsky, Martin and Chintala, Soumith and Bottou, L{\'e}on},
  booktitle={Proceedings of the 34th International Conference on Machine Learning},
  series={Proceedings of Machine Learning Research},
  volume={70},
  pages={214--223},
  year={2017},
  url={https://proceedings.mlr.press/v70/arjovsky17a.html}
}

@inproceedings{chen2018neural,
  title={Neural Ordinary Differential Equations},
  author={Chen, Ricky T. Q. and Rubanova, Yulia and Bettencourt, Jesse and Duvenaud, David},
  booktitle={Advances in Neural Information Processing Systems},
  year={2018}
}

@inproceedings{wang2020understanding,
  title={Understanding Contrastive Representation Learning through Alignment and Uniformity on the Hypersphere},
  author={Wang, Tongzhou and Isola, Phillip},
  booktitle={International Conference on Machine Learning},
  year={2020}
}

@article{papyan2020prevalence,
  title={Prevalence of Neural Collapse during the Terminal Phase of Deep Learning Training},
  author={Papyan, Vardan and Han, X. Y. and Donoho, David L.},
  journal={Proceedings of the National Academy of Sciences},
  volume={117},
  number={40},
  pages={24652--24663},
  year={2020}
}

@article{eckart1936approximation,
  title={The Approximation of One Matrix by Another of Lower Rank},
  author={Eckart, Carl and Young, Gale},
  journal={Psychometrika},
  volume={1},
  number={3},
  pages={211--218},
  year={1936}
}

@article{welch1974lower,
  title={Lower Bounds on the Maximum Cross Correlation of Signals (Corresp.)},
  author={Welch, Lloyd R.},
  journal={IEEE Transactions on Information Theory},
  volume={20},
  number={3},
  pages={397--399},
  year={1974},
  doi={10.1109/TIT.1974.1055219},
  url={https://doi.org/10.1109/TIT.1974.1055219}
}

@inproceedings{mescheder2017numerics,
  title={The Numerics of {GANs}},
  author={Mescheder, Lars and Nowozin, Sebastian and Geiger, Andreas},
  booktitle={Advances in Neural Information Processing Systems},
  volume={30},
  year={2017},
  url={https://www.nowozin.net/sebastian/papers/mescheder2017gannumerics.pdf}
}

@inproceedings{mescheder2018training,
  title={Which Training Methods for {GANs} do actually Converge?},
  author={Mescheder, Lars and Geiger, Andreas and Nowozin, Sebastian},
  booktitle={Proceedings of the 35th International Conference on Machine Learning},
  pages={3481--3490},
  volume={80},
  series={Proceedings of Machine Learning Research},
  publisher={PMLR},
  year={2018},
  url={https://proceedings.mlr.press/v80/mescheder18a.html}
}

@book{villani2009optimal,
  title={Optimal Transport: Old and New},
  author={Villani, C{\'e}dric},
  series={Grundlehren der mathematischen Wissenschaften},
  volume={338},
  publisher={Springer},
  address={Berlin, Heidelberg},
  year={2009},
  isbn={978-3-540-71049-3},
  doi={10.1007/978-3-540-71050-9},
  url={https://doi.org/10.1007/978-3-540-71050-9}
}
